\documentclass[twoside,11pt]{article}

\usepackage{blindtext}

\usepackage[abbrvbib, preprint]{jmlr2e}

\usepackage{jmlr2e}

\usepackage{adjustbox, algorithm, algpseudocode, amscd, amsfonts, amsmath, array}
\usepackage{booktabs}
\usepackage{caption, cite, color, colortbl}
\usepackage{enumerate, epstopdf}
\usepackage{fancyvrb, framed}
\usepackage{graphics}
\usepackage{hyperref}
\usepackage[utf8]{inputenc}
\usepackage{latexsym, longtable, lscape}
\usepackage{mathtools, multirow}
\usepackage{pdfpages}
\usepackage{pifont}
\usepackage{rotating}
\usepackage{tikz} 
\usepackage{url}
\usepackage{verbatim}
\usepackage{wrapfig}
\usepackage[normalem]{ulem}
\usepackage{microtype}
\usepackage{graphicx}
\usepackage{subfigure}
\usepackage[T1]{fontenc}

\newcommand{\indep}{\perp \!\!\! \perp}

\usepackage{lastpage}
\jmlrheading{}{2023}{}{8/25}{: This work is under review}{}{Young-geun Kim, Kyungbok Lee, Youngwon Choi, Joong-Ho Won, and Myunghee Cho Paik}

\ShortHeadings{Wasserstein Geodesic Generator for Conditional Distributions}{Wasserstein Geodesic Generator for Conditional Distributions}
\firstpageno{1}

\begin{document}
\sloppy
\title{Wasserstein Geodesic Generator for Conditional Distributions}

\author{\name Young-geun Kim \email Younggeun.Kim@nyspi.columbia.edu \\
       \addr Department of Biostatistics and Department of Psychiatry, Columbia University, USA
       \AND
       \name Kyungbok Lee \email turtle107@snu.ac.kr \\
       \addr Department of Statistics, Seoul, Republic of Korea
       \AND
       \name Youngwon Choi \email youngwon08@gmail.com \\
       \addr Department of Statistics, Seoul, Republic of Korea
       \AND
       \name Joong-Ho Won \email wonj@stats.snu.ac.kr \\
       \addr Department of Statistics, Seoul, Republic of Korea
       \AND
       \name Myunghee Cho Paik$^{*}$ \email myungheechopaik@snu.ac.kr \\
       \addr Department of Statistics, Seoul, Republic of Korea\\
       \addr Shepherd23 Inc., Republic of Korea}

\editor{Not determined}

\maketitle

\begin{abstract}
Generating samples given a specific label requires estimating conditional distributions. We derive a tractable upper bound of the Wasserstein distance between conditional distributions to lay the theoretical groundwork to learn conditional distributions. Based on this result, we propose a novel conditional generation algorithm where conditional distributions are fully characterized by a metric space defined by a statistical distance. We employ optimal transport theory to propose the \textit{Wasserstein geodesic generator}, a new conditional generator that learns the Wasserstein geodesic. The proposed method learns both conditional distributions for observed domains and optimal transport maps between them. The conditional distributions given unobserved intermediate domains are on the Wasserstein geodesic between conditional distributions given two observed domain labels. The proposed method generates the Wasserstein geodesic under some conditions. Experiments on face images with light conditions as domain labels demonstrate the efficacy of the proposed method.
\end{abstract}


\section{Introduction}\label{sec1}
Conditional generation is the task of constructing synthetic samples following target distributions given by specific domain labels such as age, emotion, and gender. Important applications include class-conditional image generation \citep{odena2017conditional, bao2017cvae}, age progression \citep{antipov2017face, wang2018face}, text-to-image synthesis \citep{reed2016generative, zhang2017stackgan}, and data augmentation \citep{frid2018synthetic, shao2019generative}.

Most conditional generation methods are extended from outstanding image generative models such as variational autoencoders (VAEs) \citep{kingma2014auto}, generative adversarial networks (GANs) \citep{goodfellow2014generative}, and adversarial autoencoders (AAEs) \citep{makhzani2015adversarial}. They model the distribution of images by transforming latent variables with deep neural networks.
State-of-the-art conditional generative methods include conditional VAE (cVAE) \citep{sohn2015learning}, conditional GAN (cGAN) \citep{mirza2014conditional}, and conditional AAE (cAAE) \citep{makhzani2015adversarial}. The main extension from image generative models is to concatenate domain labels into the latent variable so that the generator is a function of the latent variable as well as the domain label. A detailed review of current methods is provided in Section \hyperref[sec2]{2}.

Various conditional generative models have demonstrated realistic results for \textit{observed domains} and have usually been applied to generate samples for \textit{unobserved intermediate domains}. For example, in the age progression literature, models trained with images of people in their 20s and 50s can be applied to generate synthetic samples in unobserved intermediate domains such as 30s and 40s. Existing methods pass intermediate domain values through deep neural networks and presume that the generated data arise from the conditional distribution given the intermediate domain value. We anticipate conditional distributions given domain values change smoothly over the domain values. However, existing methods do not guarantee that conditional distributions smoothly change over the in-between regions where data are unobserved. Also, a theoretical framework describing paths on the space of conditional distributions populated by domain label values has not been provided.

The Wasserstein geodesic is the shortest path between two distributions in terms of the Wasserstein distance. We propose a novel conditional generator to learn Wasserstein geodesic, named the \textit{Wasserstein geodesic generator}. The proposed method is composed of two elements of the Wasserstein geodesic, the conditional distributions given observed domains and the optimal transport map between them, so that the conditional distributions to reside in the Wasserstein space, a metric space defined by the Wasserstein distance. The two elements are the vertices and edges of the Wasserstein geodesic in the space of the conditional distributions, respectively. For vertices, we propose a novel notion of \textit{conditional sub-coupling} for conditional generation, and adopt it to derive a tractable upper bound of the expected Wasserstein distance between the target and model conditional distributions. For edges, our proposed method learns the optimal transport map with respect to (w.r.t.) the metric on feature space specified by encoder networks. We prove that the conditional distributions given unobserved intermediate domain labels constitute the constant-speed Wasserstein geodesic between the observed domains. Our work is the first to propose conditional distributions given both observed and unobserved domains that are fully characterized by a metric space w.r.t. a statistical distance.

Our contributions are summarized as follows:
\begin{itemize}
    \item We propose a novel conditional generator that learns the Wasserstein geodesic, named the
    \textit{Wasserstein geodesic generator}. Our work is the first that can generate samples whose conditional distributions are fully characterized by the Wasserstein space.
    \item We lay a theoretical groundwork for learning conditional distributions with the Wasserstein distance by deriving a tractable upper bound of the Wasserstein distance between conditional distributions.
    \item We employ optimal transport maps between conditional distributions given two observed domains to construct the Wasserstein geodesic between the observed points in the space of conditional distributions.
    \item We derive that the proposed distribution approximates the Wasserstein barycenter in multiple observed distribution scenarios. It becomes the Wasserstein barycenter when distributions of representations are identical across observed domains.
    \item Experiments on face images with light conditions as domain labels demonstrate the efficacy of the proposed method.
\end{itemize}
The remainder of the paper is organized as follows. In Section \hyperref[sec2]{2}, we review related works including conditional generative models. Section \hyperref[sec3]{3} presents theoretical results to derive a tractable upper bound of the expected Wasserstein distance between conditional distributions and Section \hyperref[sec4]{4} introduces the proposed method. Section \hyperref[sec5]{5} presents experimental results on a face image dataset with light conditions as domain labels. All proofs of theoretical results are provided in Appendix \ref{appA}.

\section{Related Works}\label{sec2}
This section reviews related works on conditional generation, Data-to-Data Translation, and Wasserstein geometry. Our approach utilizes conditional generative models to learn observed conditional distributions, which serve as the vertices in the distribution space. We leverage Data-to-Data Translation techniques to learn intermediate paths between these observed conditional distributions, effectively establishing edges between the vertices in the distribution space. Additionally, we employ the properties of the Wasserstein space to comprehensively characterize this entire process.

The term \textit{conditional generation} encompasses various meanings, sometimes leading to confusion. In the context of our work, \textit{conditional generation} refers to a specific process of transforming latent variables
and domain labels to generate samples following conditional distributions given domain label values. We make a clear distinction between Data-to-Data Translation and conditional generative models. This distinction is made to clarify the differences between conditioning on representations and domain labels and conditioning on other observed data. It further emphasizes their orthogonal roles: conditional generative models focus on learning vertices (the distributions associated with specific domain labels), while Data-to-Data Translation concentrates on learning edges (the transitions between these distributions) in the Wasserstein space. In some works in the disentangled representation learning literature \citep{chen2016infogan, higgins2016beta, makhzani2015adversarial}, domain labels are not available and pseudo-labels are introduced to imitate domain labels.

\subsection{Conditional Generative Model}
Most conditional generative models are extended from image generation methods. We first review three eminent image generative models: VAEs, GANs, and AAEs. To synthesize realistic data, all three methods aim to learn a generator that transforms latent factors, which follow a user-specified \textit{prior distribution}. VAEs consist of encoder and decoder networks. They model the joint likelihood of latent factors and their transformations, feed-forwarded by decoder networks, and seek the maximum likelihood estimator for the marginal distribution of observations. Due to the intractability of the likelihood with nonlinear decoder networks, VAEs employ variational inference \citep{bishop2006pattern} to maximize the evidence lower bound, using encoder networks to approximate the distribution of latent variables given observations. In contrast, GANs are composed of discriminator and generator networks. The generator networks in GANs serve the same function as decoders in VAEs---transforming latent factors to generate data. However, GANs introduce discriminator networks to form an adversarial loss for generator training. AAEs, similarly, employ discriminator and generator networks but also encoders. Unlike GANs, the discriminator of AAEs aims to classify encoded results and latent variables drawn from the prior distribution. Training AAEs can be interpreted as minimizing the $1$-Wasserstein distance between distributions of real data and generation results, a special case of Wasserstein autoencoders (WAEs) \citep{tolstikhin2018wasserstein}.

Statistical distances employed in forming training objectives in generative models are pivotal components for the quality of generation results. Both VAEs and GANs utilize the $f$-divergence \citep{csiszar1964informationstheoretische} to learn distributions of images. As an alternative statistical distance, \citet{arjovsky2017wasserstein} showed that the Wasserstein distance has advantages over the $f$-divergence when the supports of data distributions are on a low-dimensional manifold as in image data. The Wasserstein distance yields differentiable losses, while $f$-divergences do not define losses well or yield non-differentiable losses, possibly due to these advantages, Wasserstein distance-based approaches including AAEs, WAEs, and Wasserstein GANs \citep{arjovsky2017wasserstein} often have outperformed $f$-divergence-based approaches.

The main extension from image generative models to conditional models is to concatenate domain labels into the latent variable. In cVAE, cGAN, and cAAE, domain labels are incorporated into encoder networks of VAE, discriminator and generators of GAN, and decoder of AAE, respectively. To enhance the visual quality and diversity of generation results, \citet{kameoka2018acvae}, \citet{odena2017conditional}, and \citet{zhao2018application} introduce auxiliary classifier to match the observed and predicted domain labels for conditional generation results. Conditional generative models can be used to generate data for unobserved domains. This is accomplished by inserting unobserved domain label values into trained models, a technique demonstrated in zero-shot learning approaches \citep{xian2018feature, chao2016empirical}, aimed at boosting classification performance for unseen classes. However, this approach faces limitations. Firstly, it assumes that latent variables for unobserved domains follow similar patterns to those for observed domains. Secondly, the generated data distribution for unobserved domains has not been justified with a metric space w.r.t. distributions. Intuitively, the unobserved intermediate distribution serves as the centroid of observed distributions, but this property has not been discussed. In contrast, our approach constructs geodesics in the Wasserstein space for unobserved domains, and the proposed distribution is the Wasserstein barycenter when distributions of representations are identical across observed domains. This result is achieved without assuming the homogeneity of representations across both observed and unobserved domains, distinguishing our approach from existing works.

We employ conditional generative models for learning data distributions from observed domain labels, observed vertices in the Wasserstein space. In Section \hyperref[sec3]{3}, we present theoretical results to provide a tractable objective for minimizing the formulated Wasserstein distances between conditional distributions.

\subsection{Data-to-Data Translation}
Conditional generative models find a transformation from latent variables and domain labels to data, effectively learning conditional distributions. In contrast, another line of work, which we refer to as Data-to-Data Translation \citep{kim2017learning, choi2018stargan, zhu2017unpaired}, operates under a different paradigm. In Data-to-Data Translation, the source and target domains are predefined, and the goal is to find a transport map from the source data to the target data. Typical examples include unpaired translations, such as converting daytime scenes to nighttime scenes, where corresponding pairs of daytime and nighttime scenes for the same location are not required during training. Prominent methods within this field include multi-modal translations across different data domains \citep{xu2018attngan, isola2017image}, such as Text-to-Image Translation, as demonstrated by DALL-E \citep{ramesh2021zero}, and multi-domain Image-to-Image Translation, as exemplified by StarGAN \citep{choi2018stargan}. 

CycleGAN \citep{zhu2017unpaired} is a pioneering work in unpaired Image-to-Image Translation. It minimizes the adversarial loss with target data and translated source data, with the cycle consistency loss encouraging the inverse relation between translation maps between two domains. In this case, the transformation is encouraged to match distributions of the target and converted source data, but it does not \textit{learn} conditional distributions given domain labels since the real data from the source domains are mandatory to construct target data. \citet{liu2017unsupervised} introduce a latent variable model to learn conditional distributions. However, the properties of the conditional distributions given unobserved intermediate domains have not been discussed in the Data-to-Data Translation literature.

Our proposed method leverages Data-to-Data Translation techniques to find optimal transport maps between observed conditional distributions, thereby defining edges between observed vertices in the Wasserstein space. In Section \hyperref[sec4]{4}, we propose to generate intermediate data from edges, Wasserstein geodesics, and extend it to approximate the centroid, which is the Wasserstein barycenter of observed distributions.

\subsection{Wasserstein Space}
The Wasserstein space, a metric space of distributions endowed with the Wasserstein distance, has found widespread applications in various fields of generative models. We review related works in image processing, domain adaptation, and data augmentation. Other important applications of Wasserstein space include density matching \citep{cisneros2020distributed}, distribution alignment \citep{zhou2022iterative}, online learning \citep{korotin2021mixability}, and Bayesian inference \citep{srivastava2018scalable}.

In image processing, most applications focus on transporting point clouds \citep{cuturi2013sinkhorn} or specific features---such as texture \citep{rabin2012wasserstein}, colors \citep{rabin2014adaptive}, and shapes \citep{solomon2015convolutional}---from a source image a target image. While these approaches have shown remarkable results, they typically require training models or solving an optimization problem for each pair of source and target images. Taking a different route, \citet{mroueh2019wasserstein} introduces universal style transfer that employs autoencoders. The method exploits the Wasserstein geodesic of Gaussian measures, using features extracted by encoder networks. However, in Mroueh's framework, distributions of generated samples for unobserved intermediate domains cannot be characterized by Wasserstein spaces. On a related note, \citet{korotin2019wasserstein} put forth a generative model to solve the dual form of the $2$-Wasserstein distance between two distributions of images, which they apply to image style transfer tasks.

In domain adaptation, \citet{xie2019scalable} proposes latent variable models using single representations to generate multiple images from each domain while minimizing the transportation costs between them. However, this method requires modality-specific generators and can not generate intermediate distributions. The Wasserstein Barycenter Transport (WBT) \citep{montesuma2021wasserstein} is a closely related work to ours. The WBT targets the Wasserstein barycenter of multiple observed source distributions to generate unobserved intermediate domains, but it requires pairs of observations from all the source domains and solves optimization problems for every generation.

In data augmentation, several recent works utilize the Wasserstein space. \citet{bespalov2022lambo} propose to augment landmark coordinates of facial images with the Wasserstein barycenter, but their method requires computing Wasserstein distances between all pairs of images to oversample landmark data. \citet{zhu2023interpolation} augment data from the Wasserstein barycenter of distributions of images to learn robust classifiers, but this method assumes the Gaussianity of conditional distributions. The work by \citet{fan2023generating} is closely related to this work. Their approach synthesizes data for unobserved domains by applying linear combinations of optimal transport maps between datasets, essentially generating data from the generalized Wasserstein geodesic of observed data distributions. Despite the merits of generalized geodesics, such as convexity and impressive performance in transfer learning tasks, the method employs the optimal transport dataset distance \citep{alvarez2020geometric}, dependent on classification labels from each domain. Additionally, the generalized geodesic differs from the Wasserstein barycenter, and the method uses an alternate transportation cost, the $(2, \nu)$-transport metric \citep{craig2016exponential} in optimization.

Existing methods typically rely on strong assumptions, such as the data following Gaussian distributions and the need to solve optimization problems each time data is generated. Other assumptions include the existence of the Wasserstein distance, optimal transport map, Wasserstein geodesics, and Wasserstein barycenters with Euclidean distances on the data space, implying that the distribution of high-dimensional data is continuous. In contrast, our work fills this gap by generating and justifying intermediate, unobserved distributions without the aforementioned assumptions.

\section{Theoretical Results on Wasserstein Distance between Conditional Distributions}\label{sec3}
\subsection{Basic Notations}\label{sec3.1}
We provide basic notations as follows. Random variables, their realizations, and their supports are denoted by capital, small, and calligraphy capital letters, respectively. The real data, generated samples, and domain labels are denoted by $X$, $\tilde{X}$, and $C$, respectively. We denote the set of distributions defined on a given support $\mathcal{X}$ by $\mathcal{P}(\mathcal{X})$ and the conditional distribution of $X$ given $Y$ by $\mathbb{P}_{X|Y}$.

For any metric $d$ on $\mathcal{X}$ and probability measures $\mathbb{P}_{X}$ and $\mathbb{P}_{Y}$ in $\mathcal{P}(\mathcal{X})$, the $p$-Wasserstein distance between $\mathbb{P}_{X}$ and $\mathbb{P}_{Y}$ w.r.t. $d$ is denoted by $W_{p}(\mathbb{P}_{X}, \mathbb{P}_{Y}; d):=\big( \underset{\pi \in \Pi(\mathbb{P}_{X}, \mathbb{P}_{Y})}{\inf} \int d^{p}(x, y)d\pi(x,y) \big)^{1/p}$ where $p \in [1, \infty)$ and $\Pi(\mathbb{P}_{X}, \mathbb{P}_{Y})$ is the set of all couplings of $\mathbb{P}_{X}$ and $\mathbb{P}_{Y}$. For brevity, we omit $d$, the metric on $\mathcal{X}$, in the Wasserstein distance if there is no confusion. We assume compactness and convexity of $\mathcal{X}$ to ensure that the Wasserstein space $(\mathcal{P}(\mathcal{X}), W_{p})$ is a geodesic space where every two points can be connected by the constant-speed geodesic \citep{santambrogio2015optimal}. 

\subsection{Distances between Conditional Distributions}\label{sec3.2}
Generating samples given a specific label requires to learn conditional distributions given domain labels. In this section, we formulate distances between target and model conditional distributions and derive a tractable upper bound of the Wasserstein distance between conditional distributions.

Denoting the latent variable independent of domain labels by $Z \sim \mathbb{P}_{Z}$, the distribution of the observed domain labels by $\mathbb{P}_{C}$, and the conditional generator by $\text{Gen}: \mathcal{Z} \times \mathcal{C} \to \mathcal{X}$, the model conditional distribution can be expressed as $\mathbb{P}_{\text{Gen}(Z,C)|C}(\cdot|c)$ where $c \in \mathcal{C}$. To learn $\mathbb{P}_{X|C}(\cdot|c)$, we formulate a class of distances between conditional distributions as
\begin{equation}\label{exp_cond_dist}
\int \mathcal{D}(\mathbb{P}_{X|C}(\cdot|c), \mathbb{P}_{\text{Gen}(Z,C)|C}(\cdot|c))d\mathbb{P}_{C}(c),
\end{equation}
where $\mathcal{D}$ is a measure between distributions. Various statistical distances can be considered for $\mathcal{D}$, and when we choose the Kullback-Leibler divergence, Equation (\ref{exp_cond_dist}) can be minimized by maximizing the expectation of the variational lower bound of the conditional log-likelihood over the distribution of domain labels. For the case of the Jensen-Shannon divergence, adversarial learning with discriminator and generator incorporating domain labels minimizes Equation (\ref{exp_cond_dist}).

To bring advantages over $f$-divergences, we focus on Equation (\ref{exp_cond_dist}) equipped with Wasserstein distances to learn the Wasserstein geodesic. Since there is no previous work formulating or deriving properties of the Wasserstein distance between conditional distributions given domain labels, in the next section, we derive an upper bound of Equation (\ref{exp_cond_dist}) that has a tractable representation.

\subsection{A Tractable Upper Bound of Wasserstein Distance between Conditional Distributions}\label{sec3.3}
In this section, we lay a theoretical groundwork by deriving a tractable upper bound of the expected Wasserstein distance between conditional distributions.

We first propose a new set of couplings for conditional generation, \textit{conditional sub-coupling}.
\begin{definition}\label{def1} For any $\mathbb{P}_{(X, C)}$ and $\mathbb{P}_{(Y, C)}$, we define the conditional sub-coupling as the set of all probability measures expressed as $\int \pi^{*}(\cdot|c)d\mathbb{P}_{C}(c)$ for some $\{\pi^{*}(\cdot|c) \}_{c \in \mathcal{C}}$ where $\pi^{*}(\cdot|c) \in \Pi(\mathbb{P}_{X|C}(\cdot|c), \mathbb{P}_{Y|C}(\cdot|c))$. The conditional sub-coupling is denoted by $\Pi(\mathbb{P}_{(X,C)}, \mathbb{P}_{(Y,C)}|\mathbb{P}_{C})$.
\end{definition}
The conditional sub-coupling is the set of all probability measures induced by couplings of conditional distributions. It is nonempty and equals to $\Pi(\mathbb{P}_{X}, \mathbb{P}_{Y})$ if $(\mathbb{P}_{X|C}(\cdot|c), \mathbb{P}_{Y|C}(\cdot|c))=(\mathbb{P}_{X}, \mathbb{P}_{Y})$ for all $c \in \mathcal{C}$. The following example provides cases where the conditional sub-coupling is a proper subset of $\Pi(\mathbb{P}_{X}, \mathbb{P}_{Y})$. Let $N(\mu_{1}, \mu_{2}, \sigma_{1}, \sigma_{2}, \rho)$ denote bivariate Gaussian distribution with mean $(\mu_{1}, \mu_{2})^{T}$ and covariance $\begin{pmatrix} \sigma_{1}^{2} & \rho\sigma_{1}\sigma_{2}\\ \rho\sigma_{1}\sigma_{2} & \sigma_{2}^{2} \end{pmatrix}$.
\begin{example}\label{ex1} Let $\mathbb{P}_{(X,C)}$ be $N(\mu_{X}, \mu_{C}, \sigma_{X}, \sigma_{C}, \rho_{XC})$ and $\mathbb{P}_{(Y, C)}$ be $N(\mu_{Y}, \mu_{C}, \sigma_{Y}, \sigma_{C}, \rho_{YC})$. Then, $\Pi(\mathbb{P}_{X}, \mathbb{P}_{Y}) \setminus \Pi(\mathbb{P}_{(X,C)}, \mathbb{P}_{(Y,C)}|\mathbb{P}_{C})$ includes $N(\mu_{X}, \mu_{Y}, \sigma_{X}, \sigma_{Y}, \rho^{*})$ if and only if $\lvert \rho^{*} - \rho_{XC}\rho_{YC} \rvert > \sqrt{(1-\rho_{XC}^{2})(1-\rho_{YC}^{2})}$.
\end{example}
Further discussions and proofs about the conditional sub-coupling are provided in Appendix \ref{appA}. With the conditional sub-coupling, we derive an upper bound of the expected $p$-Wasserstein distance in the following theorem.
\begin{theorem}\label{thm1}
Let $\mathbb{P}_{(X, C)}$ and $\mathbb{P}_{(Y, C)}$ be distributions in $\mathcal{P}(\mathcal{X} \times \mathcal{C})$. For any metric $d$ on $\mathcal{X}$ and $p \in [1, \infty)$ defining $W_{p}$,
\begin{equation}\label{eqn:upperbound}
\scalebox{0.97}{\ensuremath{\bigg( \int W_{p}^{p}(\mathbb{P}_{X|C}(\cdot|c), \mathbb{P}_{Y|C}(\cdot|c))d\mathbb{P}_{C}(c) \bigg)^{1/p} \leq \bigg( \underset{\pi^{*} \in \Pi(\mathbb{P}_{(X,C)}, \mathbb{P}_{(Y,C)}|\mathbb{P}_{C})}{\inf} \int d^{p}(x, y)d\pi^{*}(x,y) \bigg)^{1/p}.}}
\end{equation}
\end{theorem}
That is, the minimum transport cost over conditional sub-coupling is an upper bound of expected Wasserstein distance between conditional distributions.

We show a tractable representation of the upper bound in the following theorem. We denote the set of all $\mathbb{Q}_{Z|X,C}$ satisfying $\mathbb{P}_{(Z, C)}(z, c)=\big( \int \mathbb{Q}_{Z|X,C}(z|x, c)d\mathbb{P}_{X|C}(x|c) \big) \mathbb{P}_{C}(c)$ by $\mathcal{Q}$; the RHS can be considered as an aggregate posterior \citep{makhzani2015adversarial}.
\begin{theorem}\label{thm2} Let $\mathbb{P}_{(X, C)}$ and $\mathbb{P}_{(Z, C)}$ be distributions in $\mathcal{P}(\mathcal{X} \times \mathcal{C})$ and $\mathcal{P}(\mathcal{Z} \times \mathcal{C})$, respectively. For any metric space $(\mathcal{X}, d)$, $p \in [1, \infty)$, and generator $\text{Gen}: \mathcal{Z} \times \mathcal{C} \rightarrow \mathcal{X}$,
\begin{equation*}\label{eqn:representation}
\scalebox{0.89}{\ensuremath{\underset{\pi^{*} \in \Pi(\mathbb{P}_{(X,C)}, \mathbb{P}_{(\text{Gen}(Z, C),C)}|\mathbb{P}_{C})}{\inf} \int d^{p}(x, \tilde{x})d\pi^{*}(x,\tilde{x})=\underset{\mathbb{Q}_{Z|X,C} \in \mathcal{Q}}{\inf} \int  d^{p}(x, \text{Gen}(z, c)) d\mathbb{Q}_{Z|X,C}(z|x,c)d\mathbb{P}_{(X,C)}(x,c).}}
\end{equation*}
\end{theorem}
That is, the upper bound of the Wasserstein distance between conditional distributions, the RHS of Equation (\ref{eqn:upperbound}), can be expressed as the infimum of the reconstruction error over encoders $\mathbb{Q}_{Z|X,C} \in \mathcal{Q}$. Note that the integrand in the LHS of Equation (\ref{eqn:upperbound}) depends on the conditioning data $c$ and requires to evaluate the Wasserstein distance for every realization $c$, which is infeasible. In contrast, the derived representation can be computed by solving a stochastic optimization problem. When the terms related to domain label $C$ are removed, Theorem \ref{thm2} reduces to the representation of the Wasserstein distance between marginal distributions provided by \citet{tolstikhin2018wasserstein}.

\section{Proposed method}\label{sec4}
\subsection{Motivation}
For a motivating example, suppose data come from one of two observed domains whose label values are $c_{0}$ and $c_{1}$. Existing methods in conditional generation literature have considered $\text{Gen}(Z,(1-t)c_{0}+tc_{1})$ as intermediate samples \citep{zhang2017age}. However, without a strong assumption such as a linear structure, the interplay between $\mathbb{P}_{X|C}(\cdot|c_{0})$, $\mathbb{P}_{X|C}(\cdot|c_{1})$, and $\mathbb{P}_{\text{Gen}(Z,C)|C}(\cdot|(1-t)c_{0}+tc_{1})$ is difficult to formalize.

A desirable property of generated samples for unobserved intermediate domains would be their conditional distributions change smoothly from one observed domain to another. The next section proposes a new conditional generator that constructs samples from distributions on the constant-speed geodesic in the Wasserstein space.
\begin{definition}\label{def2} (Constant-speed geodesic) \citep{santambrogio2015optimal} For any $\mathbb{P}$ and $\mathbb{Q}$ on a Wasserstein space $(\mathcal{P}(\mathcal{X}), W_{p})$, a parameterized curve $w:[0,1] \to \mathcal{P}(\mathcal{X})$ is called the constant-speed geodesic from $\mathbb{P}$ to $\mathbb{Q}$ in $W_{p}$ if $w(0)=\mathbb{P}$, $w(1)=\mathbb{Q}$, and $W_{p}(w(t), w(s)) = |t-s|W_{p}(w(0), w(1))$ for any $t, s \in [0, 1]$.
\end{definition}
That is, a constant-speed geodesic in a Wasserstein space is a parameterized curve whose speed equals to the Wasserstein distance. Our method yields the conditional distribution given an unobserved intermediate domain label as an interpolation point between the conditional distributions given observed domain labels in the Wasserstein space. Unlike existing methods, the generated distributions are fully characterized by the Wasserstein space $(\mathcal{P}(\mathcal{X}), W_p)$.

\subsection{Conditional Generator for Learning Wasserstein Geodesics}\label{sec4.2}
This section proposes the \textit{Wasserstein geodesic generator}, a novel conditional generator for learning Wasserstein geodesics. Our method learns the conditional distributions given observed domains and the optimal transport maps between them to construct the Wasserstein geodesic. The proposed method consists of three networks: encoder $\text{Enc}: \mathcal{X} \times \mathcal{C} \to \mathcal{Z}$, generator $\text{Gen}: \mathcal{Z} \times \mathcal{C} \to \mathcal{X}$, and transport map $T: \mathcal{X} \times \mathcal{C}^{2} \to \mathcal{X}$.

We first define the optimal transport map, and then provide the proposed method.
\begin{definition}\label{def3} (Optimal transport map) \citep{santambrogio2015optimal} A map $T: \mathcal{X} \to \mathcal{X}$ is an optimal transport map from $\mathbb{P}_{X}$ to $\mathbb{P}_{Y}$ w.r.t. $\text{Cost}: \mathcal{X}^{2} \to \mathbb{R}$ if $T$ is a solution of the Monge-Kantorovich transportation problem, 
\begin{equation*}
\begin{aligned}
& \underset{T}{\text{minimize}} & & \int \text{Cost}(x, T(x))d\mathbb{P}_{X}(x) \\
& \text{subject to} & & \mathbb{P}_{Y}=\mathbb{P}_{T(X)}.
\end{aligned}
\end{equation*}
\end{definition}
The optimal transport map refers maps yielding the minimum transportation cost. The optimal transport map uniquely exists if the cost function is the $p$-th power of $l_{2}$-distance denoted by $\| \cdot \|$ where $p>1$ and measures are absolutely continuous on compact domains. The minimum transportation cost by the optimal transport map is known as $W^{p}_{p}(\mathbb{P}_{X}, \mathbb{P}_{Y})$.

We now present the Wasserstein geodesic generator. The proposed method postulates the encoder, generator, and transport map satisfying the following conditions $(A1)$ and $(A3)$ to generate intermediate samples. Here, $d_{\text{Enc}}((x_{1}, c_{1}), (x_{2}, c_{2})):=\lVert (\text{Enc}(x_{1}, c_{1}), c_{1})-(\text{Enc}(x_{2}, c_{2}), c_{2}) \rVert$ is a metric defined on $\mathcal{Z} \times \mathcal{C}$.
\begin{enumerate}
  \item[(A1)] (One-to-one mapping between $\mathcal{X} \times \mathcal{C}$ and $\mathcal{Z} \times \mathcal{C}$) For any $(x, c) \in \mathcal{X} \times \mathcal{C}$ and $(z, c) \in \mathcal{Z} \times \mathcal{C}$, \(\text{Gen}(\text{Enc}(x, c), c) = x\) and $\text{Enc}(\text{Gen}(z, c), c)=z$.
  \item[(A2)] (Absolutely continuous representations) For any $c \in \mathcal{C}$, $\mathbb{P}_{\text{Enc}(X, C)|C}(\cdot|c)$ is absolutely continuous, is defined on a compact set, and has finite second moments.
  \item[(A3)] (Optimal transportation) For any observed domain labels $c_{m}$ and $c_{m'}$, $(T(\cdot, \cdot, c_{m'}), c_{m'}): \mathcal{X} \times \{c_{m}\} \to \mathcal{X} \times \{ c_{m'}\}$ is the optimal transport map from $\mathbb{P}_{(X, c_{m})|C}(\cdot|c_{m})$ to $\mathbb{P}_{(X, c_{m'})|C}(\cdot|c_{m'})$ w.r.t. $d^{p}_{\text{Enc}}$.
\end{enumerate}
Note that condition $(A1)$ is about inverse relations between the encoder and generator for fixed domain labels, $(A2)$ is to guarantee the existence and uniqueness of optimal transport maps between observed conditional distributions w.r.t. $d^{p}_{\text{Enc}}$, and $(A3)$ is to build paths between observed conditional distributions with optimal transport maps.
\begin{lemma}\label{lem1} Suppose the encoder and generator satisfy the conditions $(A1)$ and $(A2)$. Then, for any $c$ and $c' \in \mathcal{C}$, $$W_{p}(\mathbb{P}_{(X, C)|C}(\cdot|c), \mathbb{P}_{(X, C)|C}(\cdot|c'); d_{\text{Enc}})=W_{p}(\mathbb{P}_{(\text{Enc}(X, C), C)|C}(\cdot|c), \mathbb{P}_{(\text{Enc}(X, C), C)|C}(\cdot|c'); \lVert \cdot \rVert).\footnote{Since marginal distributions of $\mathbb{P}_{(\text{Enc}(X,C), C)|C}(\cdot|c)$ and $\mathbb{P}_{(\text{Enc}(X,C), C)|C}(\cdot|c')$ on the second element are degenerated to $c$ and $c'$, respectively, and $(A2)$ holds, optimal transport maps between them are defined well.}$$
\end{lemma}
The encoder and generator can define a distance-preserving mapping (called \textit{isometric mapping}) between two Wasserstein spaces, and we can connect their geometric structures, including geodesics and barycenters. With Lemma \ref{lem1} and optimal transport maps, we can first generate geodesics on $(\mathcal{P}(\mathcal{Z} \times \mathcal{C}), W_{p}(\cdot, \cdot;\lVert \cdot \rVert))$, and then project them to geodesics on $(\mathcal{P}(\mathcal{X}\times\mathcal{C}), W_{p}(\cdot, \cdot;d_{\text{Enc}}))$ to generate intermediate unobserved conditional distributions.

\begin{theorem}\label{thm:speed}\label{thm3}
Suppose the encoder, generator, and transport map satisfy conditions $(A1)$ through $(A3)$. For any two observed domain labels $c_{0}$ and $c_{1}$, their convex combination $c_{t}:=(1-t)c_{0}+tc_{1}$, and $X_{0} \sim \mathbb{P}_{X|C}(\cdot|c_{0})$, the latent interpolation result of $X_{0}$ and its transported result $T(X_{0}, c_{0}, c_{1})$ can be expressed as $\text{Gen}((1-t)\text{Enc}(X_{0}, c_{0})+t\text{Enc}(T(X_{0}, c_{0}, c_{1}), c_{1}), c_{t})$. Then, the curve of distributions of latent interpolation results is the constant-speed geodesic from $\mathbb{P}_{(X, C)|C}(\cdot|c_{0})$ to $\mathbb{P}_{(X, C)|C}(\cdot|c_{1})$ in $W_{p}(\cdot, \cdot; d_{\text{Enc}})$.
\end{theorem}
The conditional distributions of the samples generated by the proposed method constitute the Wasserstein geodesic, yielding the minimum transportation cost quantified by $d_{\text{Enc}}$ between the conditional distributions of observed domains.

\subsection{Generation from the Wasserstein Barycenter with Wasserstein Geodesic Generator}\label{sec4.3}
This section extends our Wasserstein geodesic generator to accommodate scenarios involving multiple observed conditional distributions. We explain how the proposed distribution approximates the centroid of observed conditional distributions with an interpretable upper bound of the approximation error. Furthermore, we derive that the proposed distribution is the Wasserstein barycenter under some conditions.

We first define the Wasserstein barycenter, the centroid of distributions within the Wasserstein space.
\begin{definition}\label{def4} (Wasserstein barycenter) \citep{agueh2011barycenters} For any $M$ distributions $(\mathbb{P}_{m})_{m=1}^{M}$ defined on a Wasserstein space $(\mathcal{P}(\mathcal{X}), W_{2})$ and non-negative real numbers $(\alpha_{m})_{m=1}^{M}$, the Wasserstein barycenter of $(\mathbb{P}_{m})_{m=1}^{M}$ with weights $(\alpha_{m})_{m=1}^{M}$ is the unique solution of $\underset{\mathbb{P}}{\inf}\sum_{m=1}^{M}\alpha_{m}W^{2}_{2}(\mathbb{P}, \mathbb{P}_{m}).$
\end{definition}
Specifically, when $M=2$, the constant-speed geodesic $w(t)$ in Definition \ref{def2} serves as the Wasserstein barycenter of two distributions with weights $(1-t, t)$. In this context of defining a centroid between two distributions, the Wasserstein geodesic is referred to as \textit{McCann's interpolant} \citep{mccann1995existence}. The Wasserstein barycenters have been acknowledged as an effective solution for aggregating high-dimensional distributions \citep{korotin2022wasserstein}, across various applications including data augmentation \citep{huguet2022geodesic, bespalov2021data, zhu2023interpolation} and domain adaptation \citep{montesuma2021wasserstein}.

The Wasserstein barycenter has several advantages in generating unobserved conditional distributions. First, it provides smooth and stable transitions between observed distributions, which is essential for synthesizing data for new, unobserved domains. This allows the generated data to inherit characteristics from observed conditional distributions without abrupt changes. Second, it minimizes the average optimal transportation costs to observed distributions, thus minimally altering observed data to synthesize unobserved data. Last, the Wasserstein barycenter can be employed to infer the characteristics of unobserved domains, e.g., estimating the conditional average treatment effect in clinical trials \citep{huguet2022geodesic}. Despite these advantages, the computational complexity of estimating Wasserstein barycenter remains a significant bottleneck \citep{cuturi2013sinkhorn, lin2020fixed}.

We extend the Wasserstein geodesic generator to generate unobserved intermediate distributions with multiple observed distributions. In Theorem \ref{thm4}, we establish an interpretable bound on the error incurred when approximating the Wasserstein barycenter using out proposed distribution. We further demonstrate that, under condition $(A4)$ relating to the condition on the homogeneity of representations across observed domains, the proposed distribution is the Wasserstein barycenter. We begin by introducing a lemma that elucidates how the homogeneity affects on the average squared Wasserstein distances. 
\begin{lemma}\label{lem2} Suppose the encoder and generator satisfy the conditions $(A1)$ and $(A2)$. Then. for any $M$ observed domain labels $(c_{m})_{m=1}^{M}$ and their convex combination $\bar{c}:=\sum_{m=1}^{M}\alpha_{m}c_{m}$, $\underset{\mathbb{P} \in \cup_{c \in \mathcal{C}}\mathcal{P}(\mathcal{X} \times \{c\})}{\inf}\sum_{m=1}^{M}\alpha_{m}W^{2}_{2}(\mathbb{P}, \mathbb{P}_{(X, C)|C}(\cdot|c_{m}); d_{\text{Enc}})$ is equal to \begin{equation}\label{eq:minimum}
\underset{\mathbb{P} \in \mathcal{P}(\mathcal{Z})}{\inf}\sum_{m=1}^{M}\alpha_{m}W^{2}_{2}(\mathbb{P}, \mathbb{P}_{\text{Enc}(X, C)|C}(\cdot|c_{m}); \lVert \cdot \rVert) + \sum_{m=1}^{M}\alpha_{m}\lVert \bar{c}-c_{m} \rVert^{2}.
\end{equation}
holds.
\end{lemma}
In Equation (\ref{eq:minimum}), the first term represents the Wasserstein variance \citep{martinet2022variance} of $( \mathbb{P}_{\text{Enc}(X, C)|C}(\cdot|c_{m}))_{m=1}^{M}$, while the second term denotes the variance of $(c_{m})_{m=1}^{M}$ w.r.t. weights $(\alpha_{m})_{m=1}^{M}$. When multiple convex combinations are possible, we select the optimal combination that minimizes the variance, $\sum_{m=1}^{M}\alpha_{m}\lVert \bar{c}-c_{m} \rVert^{2}$. In the special case where $M=2$ and $c$ is univariate, this approach is equivalent to using the two nearest observed distributions to generate intermediate distributions. The Wasserstein variance quantifies the homogeneity of representations across observed conditional distributions and is zero if and only if the distributions of representations from all observed domains are identical, which is to say that the following condition holds.
\begin{itemize}
    \item[(A4)] $(\text{Homogeneous representations across observed domains})$ For any two observed domain labels $c_{m}$ and $c_{m'}$, $\mathbb{P}_{\text{Enc}(X, C)|C}(\cdot|c_{m})=\mathbb{P}_{\text{Enc}(X, C)|C}(\cdot|c_{m'})$.
\end{itemize}
For any $M$ observed domain labels $(c_{m})_{m=1}^{M}$, their convex combination $\bar{c}:=\sum_{m=1}^{M}\alpha_{m}c_{m}$, and $X_{1} \sim \mathbb{P}_{X|C}(\cdot|c_{1})$, the latent interpolation result of $X_{1}$ and its transported results $( T(X_{1}, c_{1}, c_{m}) )_{m=2}^{M}$ can be expressed as $\tilde{X}(\alpha_{1},\dots,\alpha_{M}):=\text{Gen}(\sum_{m=1}^{M}\alpha_{m}\text{Enc}(T(X_{1}, c_{1}, c_{m}), c_{m})), \bar{c})$. In the subsequent theorem, we derive an upper bound for the difference between the average squared Wasserstein distances with our proposed distribution and with the Wasserstein barycenter. Additionally, when we further suppose the homogeneity of representations, condition $(A4)$, the proposed method generates the Wasserstein barycenter of observed conditional distributions.
\begin{theorem}\label{thm4}Suppose the encoder, generator, and transport map satisfy conditions $(A1)$ through $(A3)$.
Then,
\begin{equation}\label{eq:bounds}
\scalebox{0.90}{\ensuremath{\begin{split}
& \sum_{m=1}^{M} \alpha_{m}W^{2}_{2}(\mathbb{P}_{(\tilde{X}(\alpha_{1},\dots,\alpha_{M}), \bar{c})}, \mathbb{P}_{(X, C)|C}(\cdot|c_{m}); d_{\text{Enc}}) - \underset{\mathbb{P} \in \cup_{c \in \mathcal{C}}\mathcal{P}(\mathcal{X} \times \{c\})}{\inf}\sum_{m=1}^{M}\alpha_{m}W^{2}_{2}(\mathbb{P}, \mathbb{P}_{(X, C)|C}(\cdot|c_{m}); d_{\text{Enc}}) \\
& \leq \frac{1}{2}\sum_{m=1}^{M}\sum_{m'=1}^{M}\alpha_{m}\alpha_{m'} \int d^{2}_{\text{Enc}}((T(x_{1}, c_{1}, c_{m'}), c_{m'}), (T(T(x_{1}, c_{1}, c_{m}), c_{m}, c_{m'}), c_{m'})) d\mathbb{P}_{X|C}(x_{1}|c_{1})
\end{split}}}
\end{equation}
holds. When we further suppose the condition $(A4)$, the upper bound is zero and $\mathbb{P}_{(\tilde{X}(\alpha_{1},\dots,\alpha_{M}), \bar{c})}$ is the Wasserstein barycenter of $(\mathbb{P}_{(X, C)|C}(\cdot|c_{m}))_{m=1}^{M}$ w.r.t. weights $(\alpha_{m})_{m=1}^{M}$.
\end{theorem}
To the best of authors' knowledge, this is the first result that justifies latent interpolation from an optimal transport point of view, without resorting to Gaussian or univariate assumptions. 
The upper bound represents half of the average squared distances between $T(X_{1}, c_{1}, c_{m'})$ and $T(T(X_{1}, c_{1}, c_{m}), c_{m}, c_{m'})$. They are identical in the univariate scenario, but not necessarily in other cases, which makes the non-negative upper bound. Although we have detailed results for the $2$-Wasserstein distance where both existence and uniqueness of barycenters are extensively examined, similar results hold for general $p$-Wasserstein distances. If $(A1)$ through $(A4)$ hold and $\underset{\mathbb{P} \in \cup_{c \in \mathcal{C}}\mathcal{P}(\mathcal{X} \times \{c\})}{\inf}\sum_{m=1}^{M} \alpha_{m}W^{2}_{p}(\mathbb{P}, \mathbb{P}_{(X, C)|C}(\cdot|c_{m}); d_{\text{Enc}})$ has solutions (called \textit{Fr{\'e}chet mean} w.r.t. $W_{p}$), then the proposed distribution is a solution. Note that condition $(A4)$ is weaker than the following condition.
\begin{itemize}
\item[(A5)] $(\text{Homogeneous representations})$ $\mathbb{P}_{\text{Enc}(X, C)|C}(\cdot|c)$ is constant w.r.t. $c \in \mathcal{C}$, equivalently,\\$\text{Enc}(X, C) \indep C$.
\end{itemize}
Most existing conditional generative models, including cVAE, cGAN, and cAAE, assume $(A5)$. When the unobserved domains possess patterns distinctly different from the observed ones, this condition may not satisfied. For example, in face frontalization \citep{huang2017beyond}, if we have facial images captured from the right and left angles and aim to synthesize frontal facial images, the frontal images might exhibit unique features, such as symmetrical facial structures, clear visibility of both eyes, and a full view of facial landmarks like the nose bridge and forehead. In the following theorem, under $(A5)$, we derive that the generation result follows the true conditional distribution.
\begin{theorem}\label{thm5}
Suppose the encoder, generator, and transport map satisfy conditions $(A1)$ through $(A5)$. Then, $\mathbb{P}_{\tilde{X}(\alpha_{1},\dots,\alpha_{M})}=\mathbb{P}_{X|C=\bar{c}}$.
\end{theorem}

\begin{algorithm}[t]\label{alg1}
\caption{Conditional generation with the Wasserstein geodesic generator}
\small
\begin{algorithmic}[1]
\State \textbf{Input:} The pair of encoder, generator, and transport map $(\text{Enc}, \text{Gen}, \text{T})$ satisfying $(A1)$ through $(A4)$, observed domain labels $(c_{m})_{m=1}^{M}$, and their convex combination $\bar{c}:=\sum_{m=1}^{M}\alpha_{m}c_{m}$.
\State \textbf{Output:} $\tilde{x}(\alpha_{1},\dots,\alpha_{M})$ following the Wasserstein barycenter of $(\mathbb{P}_{(X, C)|C}(\cdot|c_{m}))_{m=1}^{M}$ w.r.t. weights $(\alpha_{m})_{m=1}^{M}$
\State \# \textit{Conditional generation on an observed distribution (generation from a \textbf{vertex})}
\State Sample $z$ from $\mathbb{P}_{\text{Enc}(X, C)|C}(\cdot|c_{1})$
\State $\tilde{x}_{1} = \text{Gen}(z, c_{1})$
\State \# \textit{Translation to other observed domains (moving to other observed \textbf{vertices} through \textbf{edges})}
\State $\tilde{x}_{1 \to m}=T(\tilde{x}_{1}, c_{1}, c_{m})$
\State \# \textit{Latent interpolation (finding the \textbf{centroid})}
\State $\tilde{z}(\alpha_{1},\dots,\alpha_{M})=\sum_{m=1}^{M} \alpha_{m}\text{Enc}(\tilde{x}_{1 \to m}, c_{m})$
\State $\tilde{x}(\alpha_{1},\dots,\alpha_{M}) = \text{Gen}(\tilde{z}(\alpha_{1},\dots,\alpha_{M}), \bar{c})$
\end{algorithmic}
\end{algorithm}

In summary, when $(A1)$ through $(A3)$ hold, we can construct distributions that change smoothly in-between observed conditional distributions by generating geodesics. When we further suppose $(A4)$, we can generate data from the Wasserstein barycenter without observations as described in Algorithm \hyperref[alg1]{1}. Furthermore, the proposed distribution is the true conditional distribution when $(A5)$ holds.

\subsection{Implementation}\label{sec4.5}
The training of the proposed method is to learn networks satisfying conditions $(A1)$ through $(A4)$ and consists of two steps. The first step is to learn the encoder and generator pair and the second step is to learn the transport map with the learned encoder. These two steps learn the vertices and edges of the geodesic, respectively.

For the first step, motivated by Theorem \ref{thm2} and condition $(A1)$, we minimize the reconstruction error with two penalty terms
\begin{equation}\label{obj:autoencoder}
\begin{split}
    &\int \| x - \text{Gen}(\text{Enc}(x, c), c) \| ^{p} d\mathbb{P}_{X|C}(x|c)d\mathbb{P}_{C}(c)\\
    &+ \lambda_{\text{MatchLatent}} \mathcal{D}_{\text{JS}} \bigg( \mathbb{P}_{Z}(z)\mathbb{P}_{C}(c), \big( \int \delta_{z}(\text{Enc}(x,c))d\mathbb{P}_{X|C}(x|c) \big) \mathbb{P}_{C}(c) \bigg)\\
    &+ \lambda_{\text{ReconLatent}} \int \| (z, c) - (\text{Enc}(\text{Gen}(z, c), c), c) \| d\mathbb{P}_{(Z, C)}(z, c),
\end{split}
\end{equation}
where the first term is the reconstruction error, the second term is to enforce the constraint $\mathcal{Q}$ on the encoder network in the derived representation of the upper bound in Theorem \ref{thm2}, and the last term is to enforce condition $(A1)$. We substitute the deterministic encoder $\text{Enc}$ with $\mathbb{Q}_{Z|X,C}$ and $\mathcal{D}$ with $\mathcal{D}_{\text{JS}}$, and set $\mathbb{P}_{(Z, C)}=\mathbb{P}_{Z}\mathbb{P}_{C}$ to learn the Wasserstein geodesic on which the information independent of domain labels is minimally changed.
In implementation, we apply GAN for the second term and interpolation results of encoded values for the last term.

For the second step, with learned encoder from the first step, we solve Monge–Kantorovich transportation problems w.r.t. $d^{p}_{\text{Enc}}$,
\begin{equation*}\label{eqn:optimization_problem}
\begin{aligned}
& \underset{T(\cdot, c_{0}, c_{1})}{\text{minimize}} & & \int d_{\text{Enc}}^{p}((x_{0}, c_{0}), (T(x_{0}, c_{0}, c_{1}), c_{1}))d\mathbb{P}_{X|C}(x_{0}|c_{0}) \\
& \text{subject to} & & \mathbb{P}_{(X, c_{1})|C}(\cdot|c_{1})=\mathbb{P}_{(T(X, c_{0}, c_{1}), c_{1})|C}(\cdot|c_{0})
\end{aligned}
\end{equation*}
for all observed domain label values $c_{0}$ and $c_{1}$. The objective is
\begin{equation}\label{eqn:transport}
\begin{split}
    &\int d_{\text{Enc}}^{p}((x_{0}, c_{0}), (T(x_{0}, c_{0}, c_{1}), c_{1}))d\mathbb{P}_{X|C}(x_{0}|c_{0})d\mathbb{P}_{C}(c_{0})d\mathbb{P}_{C}(c_{1})\\
    &+\lambda_{\text{MatchData}} W_{p}(\mathbb{P}_{(X, C)}\mathbb{P}_{C}, \mathbb{P}_{(T(X_{0}, C_{0}, C_{1}), C_{1}, C_{0})}; \| \cdot \|)\\
    &+\lambda_{\text{Cycle}}  \int \| x_{0} - T(T(x_{0}, c_{0}, c_{1}), c_{1}, c_{0}) \| d\mathbb{P}_{X|C}(x_{0}|c_{0})d\mathbb{P}_{C}(c_{0})d\mathbb{P}_{C}(c_{1}).
\end{split}
\end{equation}
Here, $\mathbb{P}_{(T(X_{0}, C_{0}, C_{1}), C_{1}, C_{0})}$ is the distribution of $(T(X_{0}, C_{0}, C_{1}), C_{1}, C_{0})$ where $C_{0}$ and $C_{1}$ are independent samples from $\mathbb{P}_{C}$ and $X_{0}$ is sampled from $\mathbb{P}_{X|C}(\cdot|C_{0})$. The first term is the transportation cost measured by $d_{\text{Enc}}^{p}$ and the second term is to enforce constraint on conditional distributions of transported data. Note that the second term is zero if and only if $\mathbb{P}_{(X, c_{1})|C}(\cdot|c_{1})=\mathbb{P}_{(T(X, c_{0}, c_{1}), c_{1})|C}(\cdot|c_{0})$ for all $(c_{0}, c_{1})$ sampled from $\mathbb{P}_{C}\mathbb{P}_{C}$. The last term is the cycle consistency loss that encourages the inverse relation of the transport map from one domain to another and its vice versa. The cycle consistency loss has been proposed in the Data-to-Data Translation literature \citep{zhu2017unpaired, choi2018stargan} from a heuristic point of view to avoid mode collapse, but in our method, it enforces the transport map to satisfy the inverse relation of \textit{optimal} transport maps between the two domains. We derive that the minimizer of objective in Equation (\ref{eqn:transport}) is unique and is the optimal transport map between observed domains in the following theorem.
\begin{theorem}\label{thm6}
Let $T_{c_{0} \to c_{1}}^{\dag}$ be the optimal transport map from the conditional distribution given domain label  $c_{0}$ to that given $c_1$
w.r.t. $d^{p}_{\text{Enc}}$. Then, $T^{\dag}(x_{0}, c_{0}, c_{1}):= T_{c_{0} \to c_{1}}^{\dag}(x_{0})$ with probability $1$ w.r.t. $\mathbb{P}_{X|C}(x_{0}|c_{0})$ for all $c_{0}, c_{1} \in \mathcal{C}$ is the unique minimizer of objective in Equation (\ref{eqn:transport}).
\end{theorem}
In implementation, we apply WGAN with gradient penalty \citep{gulrajani2017improved} for the second term, employ an auxiliary regressor to enhance the visual quality and diversity of generated samples and add a reconstruction error, $\int d(x, T(x, c, c)) d\mathbb{P}_{(X,C)}(x,c)$, for a regularization purpose.

With the learned encoder, generator, and transport map, we can generate samples whose conditional distributions are on the Wasserstein geodesic on which the information independent of domain labels is minimally changed.

\subsection{Relationship with Other Methods}\label{sec4.4}
If either of the two steps of our training algorithm is dropped and an appropriate modification is made, our algorithm reduces to either cAAE or CycleGAN.

Suppose the second step is dropped. The first step alone learns the encoder and generator pair through minimizing objective in Equation (\ref{obj:autoencoder}). If we remove $\mathbb{P}_{C}$ in the second term and drop the third term, objective in Equation (\ref{obj:autoencoder}) reduces to that of cAAE. Thus it is the $\mathbb{P}_C$ in the second term that enables the proposed algorithm to learn the features independent of domain labels.

Now suppose the first step is removed. Since the second step optimizes objective in Equation (\ref{eqn:transport}) that includes $d_{\text{Enc}}$, it inherently depends on the first step. However, if we redefine the distance not to depend on the encoder, the second step can be operated independently. By not inheriting the $d_{\text{Enc}}$, one no longer learns a transport map while minimally changing contents independent of conditioning variables. If we replace $d_{\text{Enc}}$ in the first term with the $l_{2}$ norm and $W_{p}$ in the second term with the Jensen-Shannon divergence, and if there are only two observed domain labels $c_{0}$ and $c_{1}$, objective in Equation (\ref{eqn:transport}) reduces to that of CycleGAN. Inheriting the results from Theorem \ref{thm6}, we can establish the following corollary.
\begin{corollary}\label{cor:cycle}
Let $T_{c_{i} \to c_{j}}^{\dag}$ be the optimal transport map from the conditional distribution given domain label  $c_{i}$ to that given $c_{j}$ w.r.t. $\| \cdot \|$, where $(i, j)=(0, 1)$ or  $(1, 0)$. Then, $T^{\dag}(x_{i}, c_{i}, c_{j}):= T_{c_{i} \to c_{j}}^{\dag}(x_{i})$ with probability $1$ w.r.t. $\mathbb{P}_{X|C}(x_{i}|c_{i})$ for $(i, j)=(0, 1), (1, 0)$ is the unique minimizer of the objective of CycleGAN.
\end{corollary}
\citet{lu2019guiding} specifically point out that theoretically, there is no claim on the detailed properties of the mapping established by CycleGAN. On the other hand, \citet{korotin2019wasserstein} used the cycle consistency term to promote inverse relations between optimal transport maps, albeit with an objective based on the dual form of $2$-Wasserstein distance, which distinguishes it from CycleGAN's objective. Because CycleGAN can be viewed a special case of our proposed method, now it can be interpreted in terms of optimal transport theory, which has not been established.

\begin{figure}[h]\label{fig:generation}
\centerline{
  \begin{minipage}[b]{0.5\textwidth}
      \begin{center}
        \includegraphics[width=.95\columnwidth]{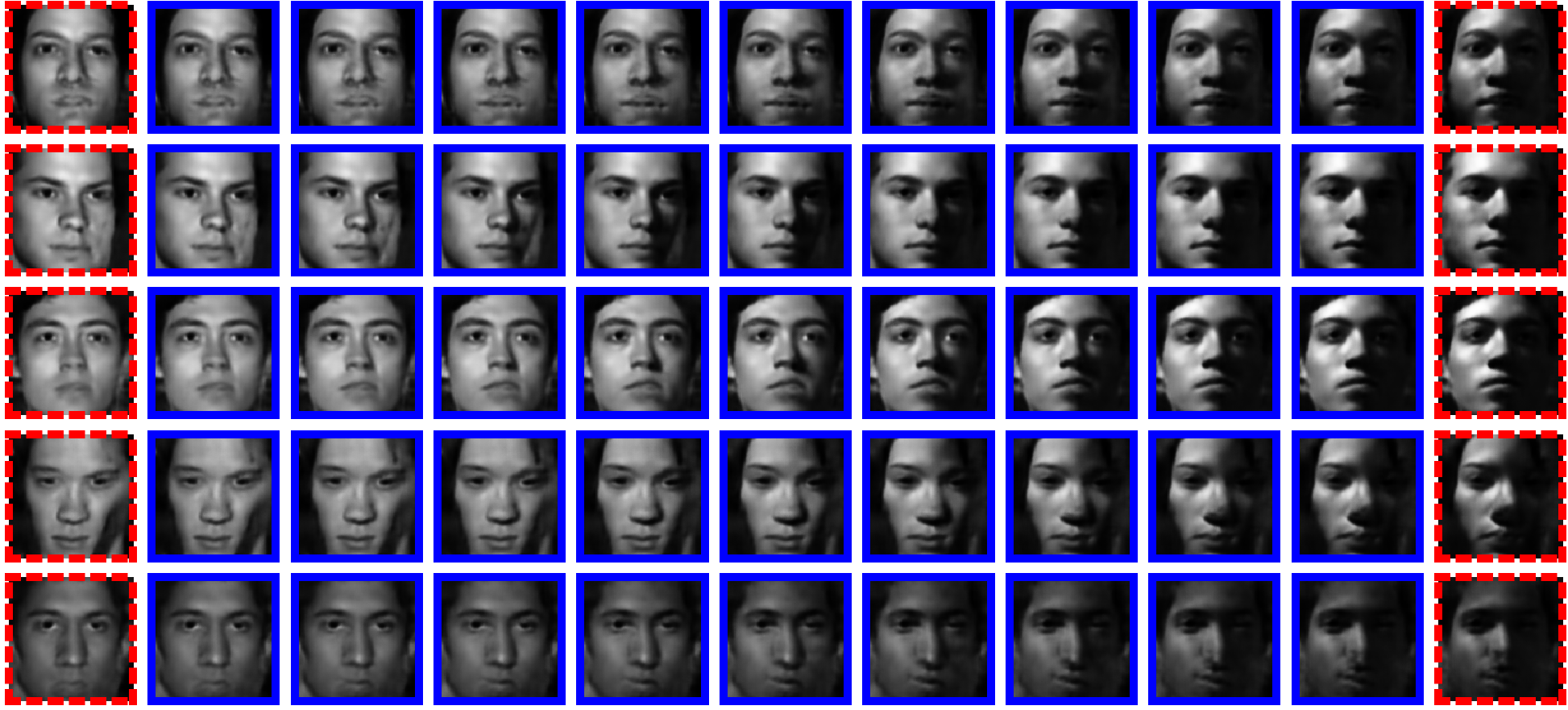}
      \end{center}
  \end{minipage}
  \begin{minipage}[b]{0.5\textwidth}
      \begin{center}
        \includegraphics[width=0.95\columnwidth]{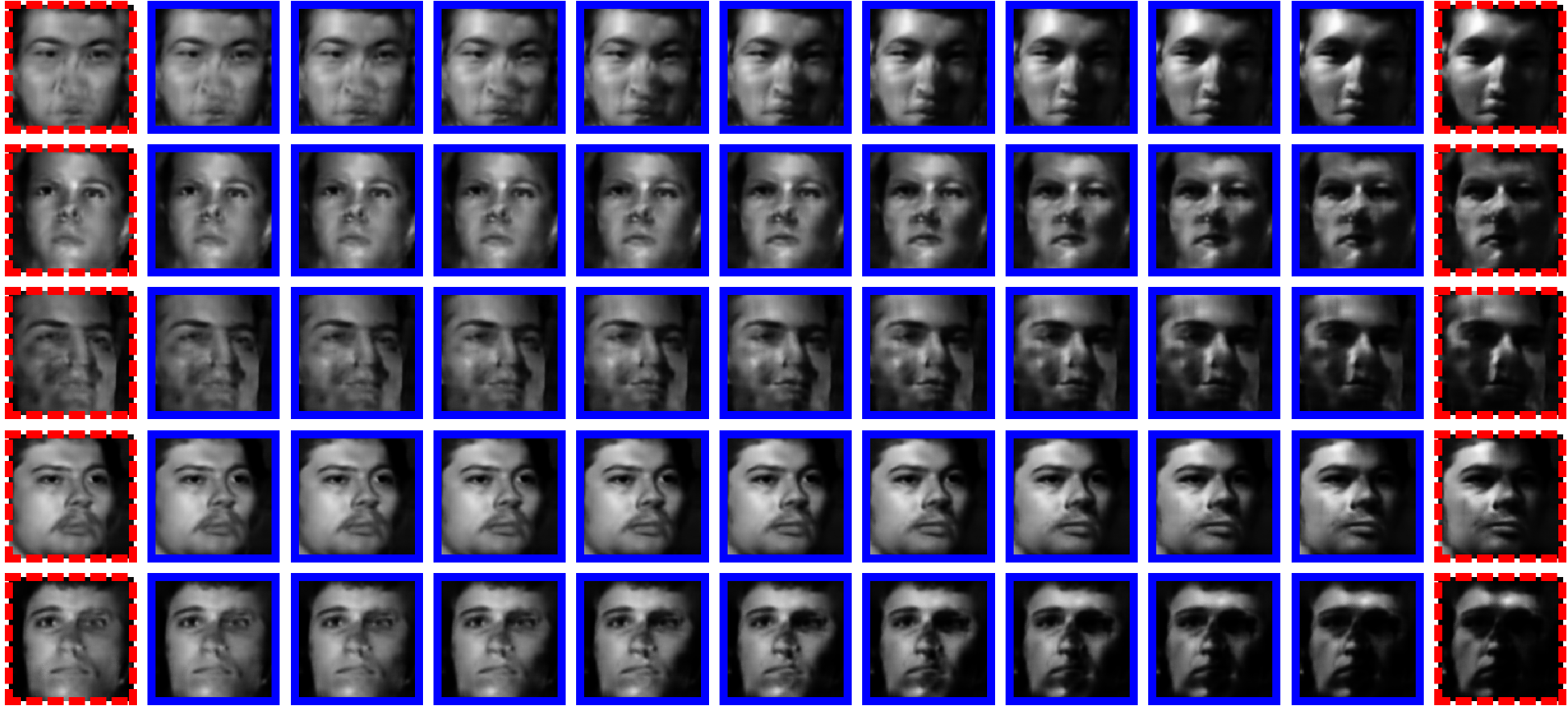}      
      \end{center}
  \end{minipage}
  }
    \caption{Conditional generation results by the proposed method (left) and cAAE (right). The proposed method produces face images with clearer eyes, noses, and mouths than baselines. For each method, the leftmost and rightmost columns show generation results for observed domains and intermediate columns show results for unobserved intermediate domains.}
\end{figure}

\begin{figure}[t]\label{fig:transportation}
\centerline{
  \begin{minipage}[b]{0.34\textwidth}
      \begin{center}
        \includegraphics[width=1.0\columnwidth]{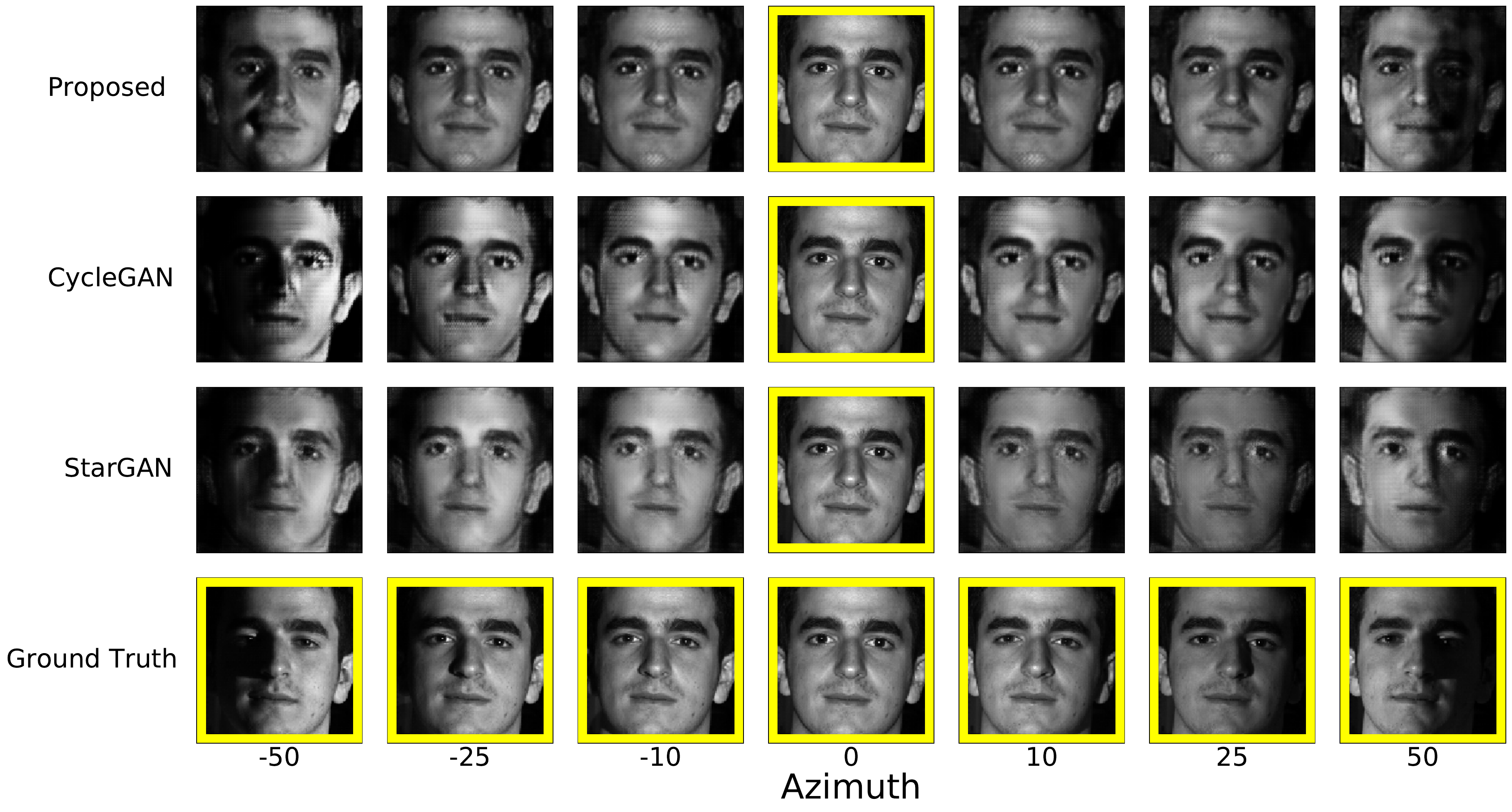}
      \end{center}
  \end{minipage}
  
  \begin{minipage}[b]{0.30\textwidth}
      \begin{center}
        \includegraphics[width=1.0\columnwidth]{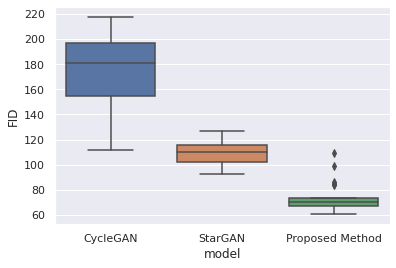}      
      \end{center}
  \end{minipage}
  
  \begin{minipage}[b]{0.36\textwidth}
      \begin{center}
        \includegraphics[width=1.0\columnwidth]{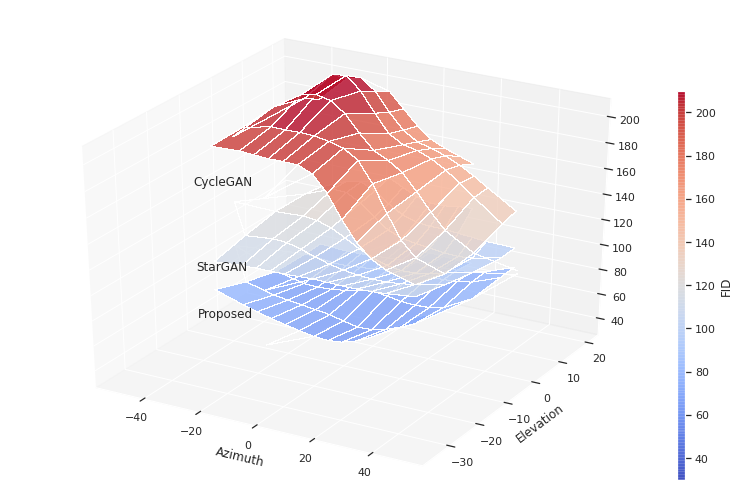}      
      \end{center}
  \end{minipage}
  }
    \caption{Comparison of transportation results by various methods. \textbf{Left:} A visualization of transportation results. The proposed method gradually casts a shadow reflecting the three-dimensional structure of the nose and mouth in the face, which makes the outcomes visually sharper and more plausible than baselines. For the first three rows, the middle column shows the real image and other columns show transportation results. \textbf{Middle:} Box plots of the FID ($\downarrow$) scores from various transport maps. The proposed method outperforms baselines. Means of FID scores of CycleGAN, StarGAN, and the proposed method with standard deviation is $174.3$ $(28.9)$, $109.7$ $(9.7)$, and $\textbf{74.2}$ $(10.9)$, respectively. \textbf{Right:} Surfaces of the FID ($\downarrow$) scores from various transport maps as functions of azimuth and elevation. In all combinations of light conditions, the proposed method outperforms baselines.}
\end{figure}

\section{Experiments} \label{sec5}
\subsection{Experimental Setting}
We conduct experiments on the Extended Yale Face Database B (Extended Yale-B) \citep{georghiades2001few} dataset. The Extended Yale-B dataset consists of face images from $29$ subjects, $10$ poses, and $64$ light conditions. We consider light conditions as domain labels. For light conditions, the azimuth and elevation of the light source are provided. The total number of images is about $15,000$.

We split the Extended Yale-B dataset by subjects to construct training, validation, and test sets. The number of subjects is $21$, $4$, and $4$ for training, validation, and test, respectively. For a data pre-processing, we apply a face detection algorithm proposed by \citet{viola2001rapid} to crop the face part.\footnote{We use the official code in OpenCV \citep{bradski2000opencv}.} The range of images is scaled to $[0, 1]$ and the horizontal flip with probability $0.5$ is applied during training.

For baselines, we consider cAAE, CycleGAN, and StarGAN \citep{choi2018stargan}. As described in Section \hyperref[sec4.5]{4.5}, our encoder and generator pair can reduce to cAAE and our transport map can reduce to CycleGAN. StarGAN is a state-of-the-art Data-to-Data Translation method and we add it as a baseline for the transport map. Algorithms of CycleGAN and StarGAN in their published works are inapplicable to continuous domain labels, so we add source and target light conditions as a input of CycleGAN and change the auxiliary classifier of StarGAN to regressor for experiments.

In all methods, architectures of the encoder and generator networks are adopted from DCGAN \citep{radford2015unsupervised}, and the architecture of the transport map is adopted from StarGAN. Architectures are modified to concatenate light conditions to latent variables, and we control the size of the networks for a fair comparison. For both the proposed method and baselines, we train the encoder and generator pair for $100,000$ iterations with batch size of $32$, and train the transport map for $50,000$ iterations with batch size of $16$. We use the Adam \citep{kingma2014adam} optimizer and set the initial learning rate to $0.0002$ and to linearly decrease to $0$ for the encoder and generator pair and to $0.0001$ for the transport map. Implementation details including architectures are provided in Appendix \ref{appB}.
\footnote{The implementation code is provided at the following link:\\ \url{http://github.com/kyg0910/Wasserstein-Geodesic-Generator-for-Conditional-Distributions}.}

For evaluation, we use the Fr{\'e}chet inception distance (FID) \citep{heusel2017gans}, a dissimilarity measure between distributions utilizing features of Inception-v3 \citep{szegedy2016rethinking}. We evaluate FID between distributions of the real and generated face images.

\subsection{Results}
We compare the proposed method and baselines with three tasks: (i) Conditional generation for unobserved intermediate domains, (ii) Data-to-Data Translation, and (iii) Latent interpolation with the real data and their translation results. Figures \hyperref[fig:generation]{1}, \hyperref[fig:transportation]{2}, and \hyperref[fig:latentinterpolation_real]{3} present results for tasks (i), (ii), and (iii), respectively. The (ii) and (iii) evaluate transport maps when an image is given. We reiterate CycleGAN and StarGAN cannot generate samples while the proposed method can.

Figure \hyperref[fig:generation]{1} presents the conditional generation results for unobserved intermediate domains. We compare the proposed method and cAAE. The proposed method produces face images with clearer eyes, noses, and mouths than baselines. For each method, the leftmost and rightmost columns show generation results for observed domains whose values of (azimuth, elevation) are $(0,0)$ and $(70,45)$, respectively, and intermediate columns show results for unobserved intermediate domains. For the proposed method, as described in Section \hyperref[sec4.2]{4.2}, we first generate samples in the leftmost column, then transport the samples to the domain of the rightmost column, and finally apply latent interpolation where $t$ increases from $0$ to $1$ at equal intervals. For cAAE, we fix latent variable $Z$ for each row and interpolate domain label $C$ with equal spacing.

Figure \hyperref[fig:transportation]{2} compares the transportation results. We compare the proposed method, CycleGAN, and StarGAN. The leftmost panel visualizes the transportation results. The proposed method gradually casts a shadow reflecting the three-dimensional structure of the nose and mouth in the face, which makes the outcomes visually sharper and more plausible than baselines. The bottom row shows face images from various azimuth for a fixed subject, pose, and elevation. The elevation is $0$ and azimuth values are shown at the bottom. For the first three rows, the middle column shows the real data in the fourth row and other columns show transportation results by various methods for observed domains corresponding to each column.
Each method translates the real data in the middle column to other observed domains. The middle panel shows the box plots of FID scores from various transport maps. The lower values are better and the proposed method outperforms baselines. The means of the FID scores of CycleGAN,  StarGAN, and the proposed method with standard deviation are 174.3 (28.9), 109.7 (9.7), and 74.2 (10.9), respectively. To calculate the FID scores, we transport real face images from (azimuth, elevation) of $(0,0)$ to other observed domains and evaluate FID between transportation results and real images for every domains. The absolute value of azimuth and elevation is less than or equal to $50$ and $35$, respectively, to remove outliers from extreme light conditions. The rightmost panel presents the surfaces of the FID scores from various transport maps as functions of azimuth and elevation. In all combinations of light conditions, the proposed method outperforms baselines. We estimate scores for unobserved domains by interpolating scores from adjacent observed domains to draw the plot. 
\begin{figure}[h]\label{fig:transportation_real}
\begin{center}
\centerline{\includegraphics[width=0.8\columnwidth]{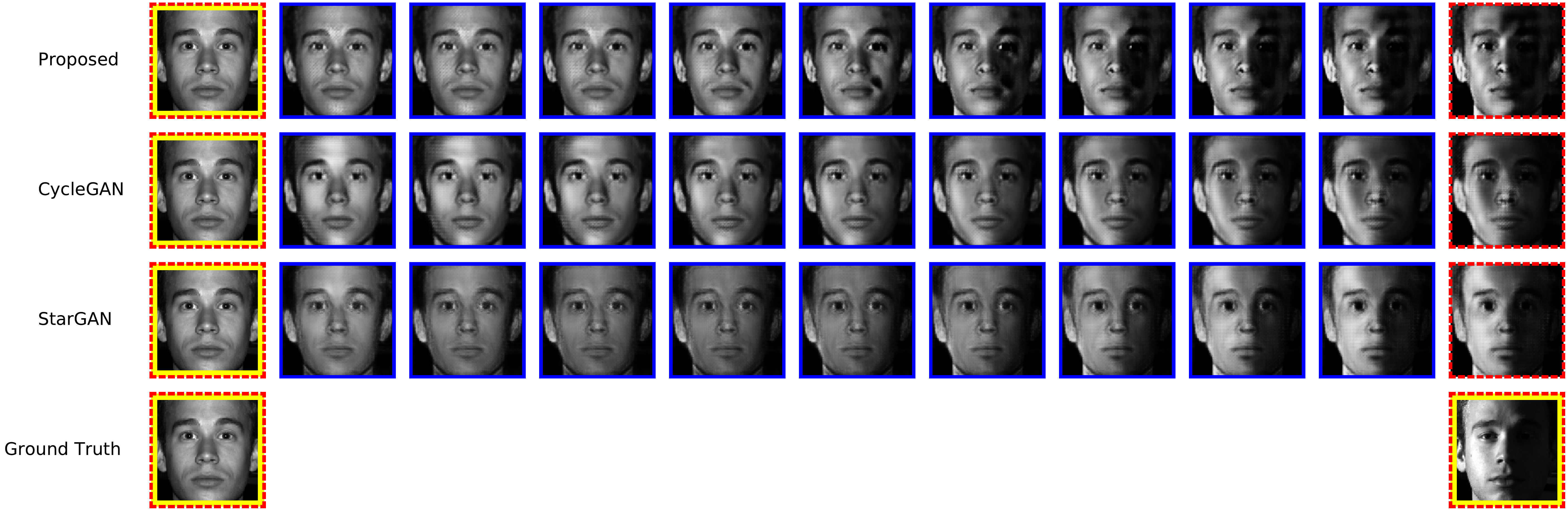}}
\end{center}
    \caption{A visualization of latent interpolation with the real data and their translation results by various methods. The bottom row shows real images from two observed domains. From the first to third rows, the leftmost column shows the ground-truth, the rightmost column shows transportation results of the ground-truth, and intermediate columns show latent interpolation results for unobserved intermediate domains.}
\label{fig:latentinterpolation_real}
\end{figure}
Figure \hyperref[fig:latentinterpolation_real]{3} presents the results of latent interpolation on unobserved intermediate domains with a real image and its translation result. We compare the proposed method, CycleGAN, and StarGAN. We fix the encoder and generator pair for all methods, and only change the transport map. The bottom row shows the ground-truth images of fixed subject and pose from two observed domains with (azimuth, elevation) of $(0, 0)$, for the leftmost, and $(70, 0)$, for the rightmost.
In the top three rows, the leftmost column shows the ground-truth, the rightmost column shows the transportation results of the ground-truth by various transport maps, and the intermediate columns show latent interpolation results for unobserved intermediate domains. As in Figure \hyperref[fig:transportation]{2}, the outputs of the proposed method are visually sharper and more plausible than baselines.

\section{Conclusion}\label{sec6}
We founded a theoretical framework for the space of conditional distributions given domain labels and proposed the \textit{Wasserstein geodesic generator}, a novel conditional generator that learns the Wasserstein geodesic. We derived a tractable upper bound of the Wasserstein distance between conditional distributions to learn those given observed domains and applied optimal transport maps between them to generate a constant-speed Wasserstein geodesic of the conditional distributions given \emph{unobserved} intermediate domains. Our work is the first to generate samples whose conditional distributions are fully characterized by a metric space w.r.t. a statistical distance. Experiments on face images with light conditions as domain labels demonstrate the efficacy of the proposed method with visually more plausible results and better FID scores than baselines.


\acks{The work of Young-geun Kim was supported by the National Research Foundation (NRF) of Korea under Grant NRF-2020R1A2C1A01011950 and by the National Institute of Mental Health of U.S. under Grant R01MH124106. The work of Kyungbok Lee and Myunghee Cho Paik was supported by the NRF under Grant NRF-2020R1A2C1A01011950.}


\newpage

\appendix
\section{Theoretical Results with Proofs}\label{appA}
\subsection{Further Discussions and Proofs about the Conditional Sub-coupling}\label{appA.1}
In this section, we discuss properties of the conditional sub-coupling, $\Pi(\mathbb{P}_{(X,C)}, \mathbb{P}_{(Y, C)}|\mathbb{P}_{C})$, and provide proofs. The following proposition describes properties of the conditional sub-coupling and corresponding properties of the upper bound of the expected Wasserstein distance between conditional distributions in Theorem \ref{thm1}.

\begin{proposition}\label{prop1} For any $\mathbb{P}_{(X, C)}$ and $\mathbb{P}_{(Y, C)}$, $\Pi(\mathbb{P}_{(X,C)}, \mathbb{P}_{(Y,C)}|\mathbb{P}_{C})$ is nonempty, included in $\Pi(\mathbb{P}_{X}, \mathbb{P}_{Y})$, and equal to $\Pi(\mathbb{P}_{X}, \mathbb{P}_{Y})$ if $(\mathbb{P}_{X|C}, \mathbb{P}_{Y|C})=(\mathbb{P}_{X}, \mathbb{P}_{Y})$. These imply that the RHS in Theorem \ref{thm1},
\begin{equation}\label{eqn:cond_wasserstein}
    \bigg( \underset{\pi ^{*}\in \Pi(\mathbb{P}_{(X,C)}, \mathbb{P}_{(Y, C)}|\mathbb{P}_{C})}{\inf}\int d^{p}(x, y)d\pi^{*}(x,y) \bigg)^{1/p},
\end{equation}
is finite, greater than or equal to $W_{p}(\mathbb{P}_{X}, \mathbb{P}_{Y})$, and equal to $W_{p}(\mathbb{P}_{X}, \mathbb{P}_{Y})$ if $(\mathbb{P}_{X|C}, \mathbb{P}_{Y|C})=(\mathbb{P}_{X}, \mathbb{P}_{Y})$.
\end{proposition}
\begin{proof}
First, the conditional sub-coupling $\Pi(\mathbb{P}_{(X,C)}, \mathbb{P}_{(Y, C)}|\mathbb{P}_{C})$ contains
$\int\mathbb{P}_{X|C}(\cdot|c) \times \mathbb{P}_{Y|C}(\cdot|c) d\mathbb{P}_{C}(c)$, so it is nonempty. Next, for any $\pi^{*}_{(X,Y)} \in \Pi(\mathbb{P}_{(X,C)}, \mathbb{P}_{(Y, C)}|\mathbb{P}_{C})$, there is $\{ \pi^{*}_{(X,Y)|C}(\cdot|c) \}_{c \in \mathcal{C}}$ such that $\pi^{*}_{(X,Y)|C}(\cdot|c) \in \Pi(\mathbb{P}_{X|C}(\cdot|c), \mathbb{P}_{Y|C}(\cdot|c))$ and $\pi^{*}_{(X,Y)}(\cdot)=\int\pi^{*}_{(X,Y)|C}(\cdot|c)d\mathbb{P}_{C}(c)$.
This implies $\int\pi^{*}_{(X,Y)}(x, y)dx=\int\pi^{*}_{(X,Y)|C}(x,y|c)dxd\mathbb{P}_{C}(c)=\int\mathbb{P}_{Y|C}(y|c)d\mathbb{P}_{C}(c)=\mathbb{P}_{Y}(y)$ and $\int\pi^{*}_{(X,Y)}(x, y)dy = \mathbb{P}_{X}(x)$. Thus, $\Pi(\mathbb{P}_{(X,C)}, \mathbb{P}_{(Y, C)}|\mathbb{P}_{C}) \subseteq \Pi(\mathbb{P}_{X}, \mathbb{P}_{Y})$ and it implies that Equation (\ref{eqn:cond_wasserstein}) is greater than or equal to $W_{p}(\mathbb{P}_{X}, \mathbb{P}_{Y})$. Finally, if $(\mathbb{P}_{X|C}, \mathbb{P}_{Y|C})=(\mathbb{P}_{X}, \mathbb{P}_{Y})$, we have $\pi(\cdot) = \int \pi(\cdot)d\mathbb{P}_{C}(c)$ and $\pi \in \Pi(\mathbb{P}_{X|C}(\cdot|c), \mathbb{P}_{Y|C}(\cdot|c))$ for all $c \in \mathcal{C}$ and $\pi \in \Pi(\mathbb{P}_{X}, \mathbb{P}_{Y})$. Thus, $\Pi(\mathbb{P}_{X}, \mathbb{P}_{Y}) \subseteq \Pi(\mathbb{P}_{(X,C)}, \mathbb{P}_{(Y, C)}|\mathbb{P}_{C})$ and it concludes the proof.
\end{proof}
The following proposition provides a detailed discussion about the relation between $\Pi(\mathbb{P}_{X}, \mathbb{P}_{Y})$ and $\Pi(\mathbb{P}_{(X,C)}, \mathbb{P}_{(Y, C)}|\mathbb{P}_{C})$, including non-Gaussian cases, by deriving a necessary condition for a coupling from $\Pi(\mathbb{P}_{X}, \mathbb{P}_{Y})$ to be in the conditional sub-coupling.
\begin{proposition}\label{prop2}
For given two distributions $\mathbb{P}_{(X, C)}$ and $\mathbb{P}_{(Y, C)}$ and a distribution $\pi \in \Pi(\mathbb{P}_{X}, \mathbb{P}_{Y})$, we denote the covariance matrix of $\pi$ by $\begin{pmatrix} \Sigma_{XX} & \Sigma^{\pi}\\ (\Sigma^{\pi})^{T} & \Sigma_{YY} \end{pmatrix}$.
Then, \vskip -0.3cm
\begin{equation}\label{positive_semi_definite}
    \begin{pmatrix} \Sigma_{XX} & \Sigma^{\pi} & \Sigma_{XC} \\ (\Sigma^{\pi})^{T} & \Sigma_{YY} & \Sigma_{YC}\\ \Sigma_{XC}^{T} & \Sigma_{YC}^{T} & \Sigma_{CC} \end{pmatrix} \text{is positive semi-definite}
\end{equation}
if $\pi \in \Pi(\mathbb{P}_{(X,C)}, \mathbb{P}_{(Y, C)}|\mathbb{P}_{C})$. Furthermore, when $\mathbb{P}_{(X,C)}$, $\mathbb{P}_{(Y,C)}$, and $\pi$ are multivariate Gaussian distributions, Equation (\ref{positive_semi_definite}) holds if and only if $\pi \in\Pi(\mathbb{P}_{(X,C)}, \mathbb{P}_{(Y, C)}|\mathbb{P}_{C})$.
\end{proposition}\vskip -0.2cm
\begin{proof}\label{proof_prop2}
By the definition of $\Pi(\mathbb{P}_{(X,C)}, \mathbb{P}_{(Y, C)}|\mathbb{P}_{C})$, for any $\pi^{*} \in \Pi(\mathbb{P}_{(X,C)}, \mathbb{P}_{(Y, C)}|\mathbb{P}_{C})$, there is $\{\pi^{*}(\cdot|c) \}_{c \in \mathcal{C}}$ such that  $\pi^{*}(\cdot)=\int \pi^{*}(\cdot|c)d\mathbb{P}_{C}(c)$ and $\pi^{*}(\cdot|c) \in \Pi(\mathbb{P}_{X|C}(\cdot|c), \mathbb{P}_{Y|C}(\cdot|c))$ for all $c \in \mathcal{C}$.
This implies the existence of random variables $(X^{*}, Y^{*}, C^{*})$ such that $(X^{*}, Y^{*})|C^{*}=c \sim \pi^{*}(\cdot|c)$ for all $c \in \mathcal{C}$ and $C^{*} \sim \mathbb{P}_{C}$.
Now, we have $(X^{*}, Y^{*}) \sim \pi^{*}$, $(X^{*}, C^{*}) \sim \mathbb{P}_{(X,C)}$, and $(Y^{*}, C^{*}) \sim \mathbb{P}_{(Y,C)}$, so the covariance matrix of $(X^{*}, Y^{*}, C^{*})$ is $\begin{pmatrix} \Sigma_{XX} & \Sigma^{\pi} & \Sigma_{XC} \\ (\Sigma^{\pi})^{T} & \Sigma_{YY} & \Sigma_{YC}\\ \Sigma_{XC}^{T} & \Sigma_{YC}^{T} & \Sigma_{CC} \end{pmatrix}$ and it should be positive semi-definite.
For the final statement in the proposition, when $\mathbb{P}_{(X,C)}$, $\mathbb{P}_{(Y,C)}$, and $\pi$ are multivariate Gaussian distributions and Equation (\ref{positive_semi_definite}) holds, we can define $(X^{**}, Y^{**}, C^{**})$ following a multivarate Gaussian distribution denoted by $N \Bigg( \begin{pmatrix} \mu_{x} \\ \mu_{y} \\ \mu_{c} \end{pmatrix}, \begin{pmatrix} \Sigma_{XX} & \Sigma^{\pi} & \Sigma_{XC} \\ (\Sigma^{\pi})^{T} & \Sigma_{YY} & \Sigma_{YC}\\ \Sigma_{XC}^{T} & \Sigma_{YC}^{T} & \Sigma_{CC} \end{pmatrix} \Bigg)$ where $\mu_{x}$, $\mu_{y}$, and $\mu_{c}$ are means of $\mathbb{P}_{X}$, $\mathbb{P}_{Y}$, and $\mathbb{P}_{C}$, respectively. Since the distribution of $(X^{**}, Y^{**})|C^{**}=c$ is in $\Pi(\mathbb{P}_{X|C}(\cdot|c), \mathbb{P}_{Y|C}(\cdot|c))$ for all $c \in \mathcal{C}$ and $(X^{**}, Y^{**}) \sim \pi$, the proof is concluded.
\end{proof}
That is, all the probability measures $\pi$ that can not be utilized to define the distribution of $(X^{*}, Y^{*}, C^{*})$ whose marginals on $(X^{*}, Y^{*})$, $(X^{*}, C^{*})$, and $(Y^{*}, C^{*})$ are $\pi$, $\mathbb{P}_{(X, C)}$ and $\mathbb{P}_{(Y, C)}$, respectively, are excluded in the conditional sub-coupling. Now, we provide Example \ref{ex1} again and provide the corresponding proof.

\noindent \textbf{Example \ref{ex1}.} \textit{Let $\mathbb{P}_{(X,C)}$ be $N(\mu_{X}, \mu_{C}, \sigma_{X}, \sigma_{C}, \rho_{XC})$ and $\mathbb{P}_{(Y, C)}$ be $N(\mu_{Y}, \mu_{C}, \sigma_{Y}, \sigma_{C}, \rho_{YC})$. Then, $\Pi(\mathbb{P}_{X}, \mathbb{P}_{Y}) \setminus \Pi(\mathbb{P}_{(X,C)}, \mathbb{P}_{(Y,C)}|\mathbb{P}_{C})$ includes $N(\mu_{X}, \mu_{Y}, \sigma_{X}, \sigma_{Y}, \rho^{*})$ if and only if $\lvert \rho^{*} - \rho_{XC}\rho_{YC} \rvert > \sqrt{(1-\rho_{XC}^{2})(1-\rho_{YC}^{2})}$.
}
\begin{proof}\label{proof_ex1}
By Proposition \ref{prop2}, it is sufficient to solve
\begin{equation*}
    \begin{vmatrix} \sigma_{X}^{2} & \rho^{*}\sigma_{X}\sigma_{Y} & \rho_{XC}\sigma_{X}\sigma_{C} \\ \rho^{*}\sigma_{X}\sigma_{Y} & \sigma_{Y}^{2} & \rho_{YC}\sigma_{Y}\sigma_{C} \\ \rho_{XC}\sigma_{X}\sigma_{C} & \rho_{YC}\sigma_{Y}\sigma_{C} & \sigma_{C}^{2} \end{vmatrix} = -\sigma_{X}^{2}\sigma_{Y}^{2}\sigma_{C}^{2}\big((\rho^{*})^{2}-2\rho_{XC}\rho_{YC}\rho^{*} + (\rho_{XC}^{2}+\rho_{YC}^{2}-1) \big) <0.
\end{equation*}
It is a quadratic inequality with respect to $\rho^{*}$, and the solution is $\rho^{*}< \rho_{XC}\rho_{YC}-\sqrt{(1-\rho_{XC}^{2})(1-\rho_{YC}^{2})}$ or $\rho^{*}> \rho_{XC}\rho_{YC}+\sqrt{(1-\rho_{XC}^{2})(1-\rho_{YC}^{2})}$.
\end{proof}
\vskip -0.5cm
\subsection{Proofs of Theoretical Results}\label{appA.2}
\subsubsection{Proof of Theorem \ref{thm1}.}\label{appA.2.1}
For any $\pi^{*} \in \Pi(\mathbb{P}_{(X,C)}, \mathbb{P}_{(Y,C)}|\mathbb{P}_{C})$, there is $\{\pi^{*}(\cdot|c)\}_{c \in \mathcal{C}}$ such that $\pi^{*}(\cdot|c) \in \Pi(\mathbb{P}_{X|C}(\cdot|c), \mathbb{P}_{Y|C}(\cdot|c))$ and $\pi^{*}(\cdot)=\int \pi^{*}(\cdot|c) d\mathbb{P}_{C}(c)$. By the definition of Wasserstein distance, $W_{p}^{p}(\mathbb{P}_{X|C}(\cdot|c), \mathbb{P}_{Y|C}(\cdot|c))\leq \int d^{p}(x, y) d\pi^{*}(x, y|c)$. This implies that $\int W_{p}^{p}(\mathbb{P}_{X|C}(\cdot|c), \mathbb{P}_{Y|C}(\cdot|c)) d\mathbb{P}_{C}(c) \leq \int d^{p}(x, y) d \big(\int \pi^{*}(\cdot|c)d\mathbb{P}_{C}(c) \big)(x, y)=\int d^{p}(x, y)d\pi^{*}(x,y)$. Now, taking infimum over all $\Pi(\mathbb{P}_{(X,C)}, \mathbb{P}_{(Y,C)}|\mathbb{P}_{C})$ concludes the proof.

\subsubsection{Proof of Theorem \ref{thm2}.}\label{appA.2.2}
We first provide a lemma to prove Theorem \ref{thm2}.

\begin{lemma}\label{lem3} For any two distributions $\mathbb{P}_{(X, C)}$ and $\mathbb{P}_{(Z,C)}$ and $G: \mathcal{Z} \times \mathcal{C} \rightarrow \mathcal{X}$, let $\tilde{\mathcal{Q}}(\mathbb{P}_{(X,C)}, \mathbb{P}_{(Z, C)}, G)$ be the set of all probability measures $\pi$ such that there exists $(X^{*}, \tilde{X}^{*}, Z^{*}, C^{*})$ satisfying $(X^{*}, \tilde{X}^{*}) \sim \pi$, $(X^{*}, C^{*}) \sim \mathbb{P}_{(X,C)}$, $\tilde{X}^{*}=G(Z^{*}, C^{*})$, and $(Z^{*}, C^{*}) \sim \mathbb{P}_{(Z,C)}$. Then, $\tilde{\mathcal{Q}}(\mathbb{P}_{(X,C)}, \mathbb{P}_{(Z, C)}, G)=\Pi(\mathbb{P}_{(X,C)}, \mathbb{P}_{(G(Z, C), C)}|\mathbb{P}_{C})$.
\end{lemma}
\begin{proof}
First, we prove $\tilde{\mathcal{Q}}(\mathbb{P}_{(X,C)}, \mathbb{P}_{(Z, C)}, G) \subseteq \Pi(\mathbb{P}_{(X,C)}, \mathbb{P}_{(G(Z, C), C)}|\mathbb{P}_{C})$. By definition, for any $\pi \in \tilde{\mathcal{Q}}(\mathbb{P}_{(X,C)}, \mathbb{P}_{(Z, C)}, G)$, there exists $(X^{*}, \tilde{X}^{*}, Z^{*}, C^{*})$ satisfying $(X^{*}, \tilde{X}^{*}) \sim \pi$, $(X^{*}, C^{*}) \sim \mathbb{P}_{(X,C)}$, $\tilde{X}^{*}=G(Z^{*}, C^{*})$, and $(Z^{*}, C^{*}) \sim \mathbb{P}_{(Z,C)}$. Since $X^{*}|C^{*} \sim \mathbb{P}_{X|C}(\cdot|C=C^{*})$ and $\tilde{X}^{*}|C^{*} \sim \mathbb{P}_{G(Z,C)|C}(\cdot|C=C^{*})$, the distribution of $(X^{*}, \tilde{X}^{*})|C^{*}$ is an element of $\Pi(\mathbb{P}_{X|C}(\cdot|C=C^{*}), \mathbb{P}_{G(Z,C)|C}(\cdot|C=C^{*}))$. Since $C^{*} \sim \mathbb{P}_{C}$, it is shown that $\pi \in \Pi(\mathbb{P}_{(X,C)}, \mathbb{P}_{(G(Z, C), C)}|\mathbb{P}_{C})$.

Next, we prove $\tilde{\mathcal{Q}}(\mathbb{P}_{(X,C)}, \mathbb{P}_{(Z, C)}, G) \supseteq \Pi(\mathbb{P}_{(X,C)}, \mathbb{P}_{(G(Z, C), C)}|\mathbb{P}_{C})$. For any $\pi \in \Pi(\mathbb{P}_{(X,C)}, \mathbb{P}_{(G(Z, C), C)}|\mathbb{P}_{C})$, there exists $(X^{*}, \tilde{X}^{*}, C^{*})$ such that $(X^{*}, \tilde{X}^{*}) \sim \pi$, $X^{*}|C^{*} \sim \mathbb{P}_{X|C}(\cdot|C=C^{*})$, $\tilde{X}^{*}|C^{*} \sim \mathbb{P}_{G(Z, C)|C}(\cdot|C=C^{*})$, and  $C^{*} \sim \mathbb{P}_{C}$. Since $\tilde{X}^{*}|C^{*} \sim \mathbb{P}_{G(Z, C)|C}(\cdot|C=C^{*})$, there exists $Z^{*}$ such that $Z^{*}|C^{*} \sim \mathbb{P}_{Z|C}(\cdot|C=C^{*})$ and $\tilde{X}^{*} = G(Z^{*}, C^{*})$. As $(X^{*}, \tilde{X}^{*}, Z^{*}, C^{*})$ satisfies $(X^{*}, \tilde{X}^{*}) \sim \pi$, $(X^{*}, C^{*}) \sim \mathbb{P}_{(X,C)}$, $\tilde{X}^{*}=G(Z^{*}, C^{*})$, and $(Z^{*}, C^{*}) \sim \mathbb{P}_{(Z,C)}$, it is shown that $\pi \in \tilde{\mathcal{Q}}(\mathbb{P}_{(X,C)}, \mathbb{P}_{(Z, C)}, G)$, which concludes the proof.
\end{proof}
Now, we prove Theorem \ref{thm2}. Showing that $\underset{\pi \in \tilde{\mathcal{Q}}(\mathbb{P}_{(X,C)}, \mathbb{P}_{(Z, C)}, G)}{\inf} \int d^{p}(x, \tilde{x})d\pi(x, \tilde{x})$ is equal to $\underset{\mathbb{Q}_{Z|X,C} \in \mathcal{Q}}{\inf} \int  d^{p}(x, \text{Gen}(z, c)) d\mathbb{Q}_{Z|X,C}(z|x,c)d\mathbb{P}_{(X,C)}(x,c)$ is sufficient by Lemma \ref{lem1} where $\mathcal{Q}$ denotes the set of all $\mathbb{Q}_{Z|X,C}$ satisfying $\mathbb{P}_{(Z, C)}(z, c)=\big( \int \mathbb{Q}_{Z|X,C}(z|x, c)d\mathbb{P}_{X|C}(x|c) \big) \mathbb{P}_{C}(c)$.

First, we show that LHS $\geq$ RHS. For any $\pi \in \tilde{\mathcal{Q}}(\mathbb{P}_{(X,C)}, \mathbb{P}_{(Z, C)}, G)$, there exists $(X^{*}, \tilde{X}^{*}, Z^{*}, C^{*})$ satisfying $(X^{*}, \tilde{X}^{*}) \sim \pi$, $(X^{*}, C^{*}) \sim \mathbb{P}_{(X,C)}$, $\tilde{X}^{*}=G(Z^{*}, C^{*})$, and $(Z^{*}, C^{*}) \sim \mathbb{P}_{(Z,C)}$. This implies that
\begin{equation*}
    \begin{split}
        \int d^{p}(x, \tilde{x})d\pi(x, \tilde{x}) &=\mathbb{E}_{(X^{*}, \tilde{X}^{*}, Z^{*}, C^{*})}d^{p}(x^{*}, \tilde{x}^{*})\\
        &=\mathbb{E}_{(Z^{*}, C^{*})} \mathbb{E}_{(X^{*}, \tilde{X}^{*})|Z^{*}, C^{*}} \big( d^{p}(X^{*}, \tilde{X}^{*}) |Z^{*}, C^{*} \big)\\
        &=\mathbb{E}_{(Z^{*}, C^{*})} \mathbb{E}_{(X^{*}, G(Z^{*}, C^{*}))|Z^{*}, C^{*}} \big( d^{p}(X^{*}, G(Z^{*}, C^{*}))  |Z^{*}, C^{*} \big)\\
        &=\mathbb{E}_{(Z^{*}, C^{*})} \mathbb{E}_{X^{*}|Z^{*}, C^{*}} \big( d^{p}(X^{*}, G(Z^{*}, C^{*}))  |Z^{*}, C^{*} \big)\\
        &=\mathbb{E}_{(X^{*}, C^{*})} \mathbb{E}_{Z^{*}|X^{*}, C^{*}} \big( d^{p}(X^{*}, G(Z^{*}, C^{*}))  |X^{*}, C^{*} \big)\\
        &=\int  d^{p}(x^{*}, \text{Gen}(z^{*}, c^{*})) d\mathbb{Q}_{Z^{*}|X^{*},C^{*}}(z^{*}|x^{*},c^{*})d\mathbb{P}_{(X^{*},C^{*})}(x^{*},c^{*}).
    \end{split}
\end{equation*}
We denote the distribution of $Z^{*}|X^{*}, C^{*}$ by $\mathbb{Q}_{Z^{*}|X^{*}, C^{*}}$. Since $Z^{*}|C^{*} \sim \mathbb{P}_{Z|C}(\cdot|C=C^{*})$ and $X^{*}|C^{*} \sim \mathbb{P}_{X|C}(\cdot|C=C^{*})$, $\mathbb{P}_{Z|C}(z^{*}|c^{*})=\int \mathbb{Q}_{Z^{*}|X^{*}, C^{*}}(z^{*}|x^{*},c^{*})d\mathbb{P}_{X|C}(x^{*}|c^{*})$ for all $x^{*} \in \mathcal{X}$, $z^{*} \in \mathcal{Z}$, and $c^{*} \in \mathcal{C}$.
Thus, $\mathbb{Q}_{Z^{*}|X^{*}, C^{*}} \in \mathcal{Q}$. This and $(X^{*}, C^{*}) \sim \mathbb{P}_{(X, C)}$ imply $\int d^{p}(x, \tilde{x})d\pi(x, \tilde{x}) \geq \underset{\mathbb{Q}_{Z|X,C} \in \mathcal{Q}}{\inf} \int  d^{p}(x, \text{Gen}(z, c)) d\mathbb{Q}_{Z|X,C}(z|x,c)d\mathbb{P}_{(X,C)}(x,c)$. It concludes LHS $\geq$ RHS.

Next, we show LHS $\leq$ RHS. For any $\mathbb{Q}_{Z|X,C} \in \mathcal{Q}$, there exists $(X^{*}, Z^{*}, C^{*})$ satisfying $Z^{*}|X^{*}, C^{*} \sim \mathbb{Q}_{Z|X,C}(\cdot|X=X^{*}, C=C^{*})$, $(X^{*}, C^{*}) \sim \mathbb{P}_{(X, C)}$, and $(Z^{*}, C^{*}) \sim \mathbb{P}_{(Z, C)}$. We denote $\tilde{X}^{*}:=G(Z^{*}, C^{*})$ and the distribution of $(X^{*}, \tilde{X}^{*})$ by $\pi^{*}$. Then,
\begin{equation*}
\scalebox{0.88}{\ensuremath{\begin{split}
    \int  d^{p}(x^{*}, \text{Gen}(z^{*}, c^{*})) d\mathbb{Q}_{Z|X,C}(z^{*}|x^{*},c^{*})d\mathbb{P}_{(X,C)}(x^{*},c^{*}) &=\mathbb{E}_{(X^{*}, C^{*})} \mathbb{E}_{Z^{*}|X^{*}, C^{*}} \big( d^{p}(X^{*}, G(Z^{*}, C^{*})) |X^{*}, C^{*} \big)\\
    &=\mathbb{E}_{(Z^{*}, C^{*})} \mathbb{E}_{X^{*}|Z^{*}, C^{*}} \big( d^{p}(X^{*}, G(Z^{*}, C^{*})) |Z^{*}, C^{*} \big)\\
    &=\mathbb{E}_{(Z^{*}, C^{*})} \mathbb{E}_{(X^{*}, \tilde{X}^{*})|Z^{*}, C^{*}} \big( d^{p}(X^{*}, \tilde{X}^{*}) |Z^{*}, C^{*} \big)\\
    &=\mathbb{E}_{(X^{*}, \tilde{X}^{*}, Z^{*}, C^{*})} d^{p}(X^{*}, \tilde{X}^{*})\\
    &= \int d^{p}(x^{*}, \tilde{x}^{*}) d\pi^{*}(x, \tilde{x}^{*}).
\end{split}}}
\end{equation*}
Here, $\pi^{*} \in \tilde{\mathcal{Q}}(\mathbb{P}_{(X,C)}, \mathbb{P}_{(Z, C)}, G)$ because $(X^{*}, \tilde{X}^{*}, Z^{*}, C^{*})$ satisfies $(X^{*}, \tilde{X}^{*}) \sim \pi^{*}$, $(X^{*}, C^{*}) \sim \mathbb{P}_{(X,C)}$, $\tilde{X}^{*}=G(Z^{*}, C^{*})$, and $(Z^{*}, C^{*}) \sim \mathbb{P}_{(Z,C)}$. Thus, we have $\underset{\pi \in \tilde{\mathcal{Q}}(\mathbb{P}_{(X,C)}, \mathbb{P}_{(Z, C)}, G)}{\inf} \int d^{p}(x, \tilde{x})d\pi(x, \tilde{x}) \leq \int  d^{p}(x, \text{Gen}(z, c)) d\mathbb{Q}_{Z|X,C}(z|x,c)d\mathbb{P}_{(X,C)}(x,c)$, which concludes the proof.

\subsubsection{Proof of Lemma \ref{lem1}.}\label{appA.2.3}
For any $\mathbb{Q}_{(X,C,X'C')}$ s.t. $\mathbb{Q}_{(X,C)}=\mathbb{P}_{(X,C)|C}(\cdot|c)$ and $\mathbb{Q}_{(X',C')}=\mathbb{P}_{(X,C)|C}(\cdot|c')$, by the definition of the Wasserstein distance,
\begin{equation*}
\scalebox{0.86}{\ensuremath{\begin{split}
\int d^{p}_{\text{Enc}}((x, c), (x', c'))d\mathbb{Q}_{(X,C,X',C')}(x,c,x',c') &= \int \lVert (\text{Enc}(x, c), c) - (\text{Enc}(x', c'), c') \rVert^{p} d\mathbb{Q}_{(X,C,X',C')}(x,c,x',c') \\
& \geq W^{p}_{p}(\mathbb{P}_{(\text{Enc}(X, C), C)|C}(\cdot|c), \mathbb{P}_{(\text{Enc}(X, C), C)|C}(\cdot|c'); \lVert \cdot \rVert).
\end{split}}}
\end{equation*}
This implies that $W_{p}(\mathbb{P}_{(X,C)|C}(\cdot|c), \mathbb{P}_{(X,C)|C}(\cdot|c'); d_{\text{Enc}})$ is greater than or equal to $W_{p}(\mathbb{P}_{(\text{Enc}(X, C), C)|C}(\cdot|c), \mathbb{P}_{(\text{Enc}(X, C), C)|C}(\cdot|c'); \lVert \cdot \rVert)$. Similarly, for any $\mathbb{Q}_{(Z,C,Z',C')}$ s.t. $\mathbb{Q}_{(Z, C)}=\mathbb{P}_{(\text{Enc}(X, C), C)|C}(\cdot|c)$ and $\mathbb{Q}_{(Z', C')}=\mathbb{P}_{(\text{Enc}(X, C), C)|C}(\cdot|c')$, 
\begin{equation*}
\scalebox{0.87}{\ensuremath{\begin{split}
\int \lVert (z, c) - (z', c') \rVert^{p} d\mathbb{Q}_{(Z, C, Z', C')}(z, c, z', c') &= \int d^{p}_{\text{Enc}}( (\text{Gen}(z, c), c), (\text{Gen}(z', c'), c') ) d\mathbb{Q}_{(Z, C, Z', C')}(z, c, z', c') \\
&\geq W^{p}_{p}(\mathbb{P}_{(X, C)|C}(\cdot|c), \mathbb{P}_{(X, C)|C}(\cdot|c'); d_{\text{Enc}})
\end{split}}}
\end{equation*}
holds by $(A1)$ and the definition of the Wasserstein distance, which implies $W_{p}(\mathbb{P}_{(\text{Enc}(X, C), C)|C}(\cdot|c), \mathbb{P}_{(\text{Enc}(X, C), C)|C}(\cdot|c'); \lVert \cdot \rVert) \geq W_{p}(\mathbb{P}_{(X, C)|C}(\cdot|c), \mathbb{P}_{(X, C)|C}(\cdot|c'); d_{\text{Enc}})$.

\subsubsection{Proof of Theorem \ref{thm3}.}\label{appA.2.4}
By Lemma \ref{lem1}, it is sufficient to show that the distribution of $((1-t)\text{Enc}(X_{0}, c_{0})+t\text{Enc}(T(X_{0}, c_{0}, c_{1}), c_{1}), (1-t)c_{0}+tc_{1})$, denoted by $w(t)$, is the constant-speed Wasserstein geodesic from $\mathbb{P}_{\text{Enc}(X, C)|C}(\cdot|c_{0})$ to $\mathbb{P}_{\text{Enc}(X, C)|C}(\cdot|c_{1})$ in $W_{p}(\cdot, \cdot; \lVert \cdot \rVert)$. First, by $(A2)$, $w(0):=\mathbb{P}_{(\text{Enc}(X_{0}, c_{0}), c_{0})}=\mathbb{P}_{(\text{Enc}(X, c_{0}), c_{0})}(\cdot|c_{0})=\mathbb{P}_{(\text{Enc}(X, C), C)|C}(\cdot|c_{0})$. By $(A3)$, $\mathbb{P}_{(T(X,C,c_{1}), c_{1})|C}(\cdot|c_{0})=\mathbb{P}_{(X, C)|C}(\cdot|c_{1})$, which implies $\mathbb{P}_{(\text{Enc}(T(X,C,c_{1}), c_{1}), c_{1})|C}(\cdot|c_{0})=\mathbb{P}_{(\text{Enc}(X, c_{1}), C)|C}(\cdot|c_{1})=\mathbb{P}_{(\text{Enc}(X, C), C)|C}(\cdot|c_{1})$. This and $(A2)$ imply that $w(1):=\mathbb{P}_{(\text{Enc}(T(X_{0}, c_{0}, c_{1}), c_{1}), c_{1})}=\mathbb{P}_{(\text{Enc}(T(X, C, c_{1}), c_{1}), c_{1})|C}(\cdot|c_{0})=\mathbb{P}_{(\text{Enc}(X, C), C)|C}(\cdot|c_{1})$. Next, for any $t \leq t'$, by the definition of the Wasserstein distance, $(A3)$, and Lemma \ref{lem1}, we have
\begin{equation*}
\scalebox{0.92}{\ensuremath{
    \begin{aligned}
        & W_{p}(w(t), w(t'); \lVert \cdot \rVert)&\leq& \Big(\int \lVert ((1-t)\text{Enc}(x_{0}, c_{0})+t\text{Enc}(T(x_{0}, c_{0}, c_{1}), c_{1}), \tilde{c}_{t}) - \\
        & & &((1-t')\text{Enc}(x_{0}, c_{0})+t'\text{Enc}(T(x_{0}, c_{0}, c_{1}), c_{1}), \tilde{c}_{t'}) \rVert^{p} d\mathbb{P}_{X|C}(x_{0}|c_{0}) \Big)^{1/p}\\
        & & =&|t-t'|\Big( \int \| (\text{Enc}(x_{0}, c_{0}), c_{0})-(\text{Enc}(T(x_{0}, c_{0}, c_{1}), c_{1}), c_{1}) \|^{p} d\mathbb{P}_{X|C}(x_{0}|c_{0}) \Big)^{1/p}\\
        & & =&|t-t'|\Big( \int d^{p}_{\text{Enc}}((x_{0}, c_{0}), (T(x_{0}, c_{0}, c_{1}), c_{1})) d\mathbb{P}_{X|C}(x_{0}|c_{0}) \Big)^{1/p}\\
        & & =&|t-t'| W_{p}(\mathbb{P}_{(X, C)|C}(\cdot|c_{0}), \mathbb{P}_{(X, C)|C}(\cdot|c_{1}); d_{\text{Enc}}) \\
        & & =&|t-t'| W_{p}(\mathbb{P}_{(\text{Enc}(X, C), C)|C}(\cdot|c_{0}), \mathbb{P}_{(\text{Enc}(X, C), C)|C}(\cdot|c_{1}); \lVert \cdot \rVert).&
    \end{aligned}}}
\end{equation*}
Here, the last term is $|t-t'|W_{p}(w(0), w(1); \lVert \cdot \rVert)$ and it implies
\begin{equation*}
    \begin{split}
        W_{p}(w(0), w(1); \lVert \cdot \rVert)&\leq W_{p}(w(0), w(t); \lVert \cdot \rVert + W_{p}(w(t), w(t'); \lVert \cdot \rVert + W_{p}(w(t'), w(1); \lVert \cdot \rVert)\\
        & \leq (t+(t'-t)+(1-t'))W_{p}(w(0), w(1); \lVert \cdot \rVert)\\
        & = W_{p}(w(0), w(1); \lVert \cdot \rVert).
    \end{split}
\end{equation*}
Thus, equalities hold and it concludes the proof.

\subsubsection{Proof of Lemma \ref{lem2}.}\label{appA.2.5}
By Lemma \ref{lem1}, showing that $\underset{\mathbb{P} \in \cup_{c \in \mathcal{C}}\mathcal{P}(\mathcal{Z} \times \{c\})}{\inf}\sum_{m=1}^{M}\alpha_{m}W^{2}_{2}(\mathbb{P}, \mathbb{P}_{(\text{Enc}(X, C), C)|C}(\cdot|c_{m}); \lVert \cdot \rVert)$ is equal to $\underset{\mathbb{P} \in \mathcal{P}(\mathcal{Z})}{\inf}\sum_{m=1}^{M}\alpha_{m}W^{2}_{2}(\mathbb{P}, \mathbb{P}_{\text{Enc}(X, C)|C}(\cdot|c_{m}); \lVert \cdot \rVert) + \sum_{m=1}^{M}\alpha_{m}\lVert \bar{c}-c_{m} \rVert^{2}$ is sufficient. By Proposition 4.2 in \citeauthor{agueh2011barycenters} (\citeyear{agueh2011barycenters}), for any measures $( \mathbb{P}_{m})_{m=1}^{M}$ vanishing on small sets,\footnote{A measure defined on $D$-dimensional spaces is said to vanish on small sets if it vanishes on $(D-1)$-rectifiable sets \citep{gangbo1998optimal, agueh2011barycenters}.}
\begin{equation}\label{eq:dual_multimarginal}
\underset{\mathbb{P}}{\inf} \sum_{m=1}^{M}\alpha_{m}W^{2}_{2}(\mathbb{P}, \mathbb{P}_{m}; \lVert \cdot \rVert)=  \underset{\pi \in \Pi(\mathbb{P}_{1}, \dots, \mathbb{P}_{M})}{\inf} \int \sum_{m=1}^{M}\alpha_{m} \lVert x_{m} - \sum_{m=1}^{M}\alpha_{m}x_{m} \rVert^{2} d\pi(x_{1},\dots,x_{M})
\end{equation}
and $\mathbb{P}_{X^{*}}=\mathbb{P}_{\sum_{m=1}^{M}\alpha_{m}X^{*}_{m}}$ hold where $\Pi(\mathbb{P}_{1}, \dots, \mathbb{P}_{M})$ is the set of joint distributions whose marginals are $( \mathbb{P}_{m})_{m=1}^{M}$, $\mathbb{P}_{X^{*}}$ is the unique solution of the LHS, and $\mathbb{P}_{(X^{*}_{1},\dots,X^{*}_{M})}$ is the unique solution of the RHS. By Equation (\ref{eq:dual_multimarginal}),
\begin{equation*}
\scalebox{0.83}{\ensuremath{
\begin{split}
&\underset{\mathbb{P} \in \cup_{c \in \mathcal{C}}\mathcal{P}(\mathcal{Z} \times \{c\})}{\inf}\sum_{m=1}^{M}\alpha_{m}W^{2}_{2}(\mathbb{P}, \mathbb{P}_{(\text{Enc}(X, C), C)|C}(\cdot|c_{m}); \lVert \cdot \rVert)\\
&= \underset{\mathbb{Q}_{(Z_{1},\dots,Z_{M})} \in \Pi(\mathbb{P}_{\text{Enc}(X,C)|C}(\cdot|c_{1}),\dots,\mathbb{P}_{\text{Enc}(X,C)|C}(\cdot|c_{M}))}{\inf}\int \sum_{m=1}^{M}\alpha_{m} \lVert (z_{m}, c_{m}) - \sum_{m=1}^{M}\alpha_{m}(z_{m}, c_{m}) \rVert^{2} d\mathbb{Q}_{(Z_{1},\dots,Z_{M})}(z_{1},\dots,z_{M})\\
&= \underset{\mathbb{Q}_{(Z_{1},\dots,Z_{M})} \in \Pi(\mathbb{P}_{\text{Enc}(X,C)|C}(\cdot|c_{1}),\dots,\mathbb{P}_{\text{Enc}(X,C)|C}(\cdot|c_{M}))}{\inf}\int \sum_{m=1}^{M}\alpha_{m} \lVert z_{m} - \sum_{m=1}^{M}\alpha_{m}z_{m} \rVert^{2} d\mathbb{Q}_{(Z_{1},\dots,Z_{M})}(z_{1},\dots,z_{M})\\
&\quad + \sum_{m=1}^{M} \alpha_{m} \lVert \bar{c} - c_{m} \rVert^{2}\\
&=\underset{\mathbb{P} \in \mathcal{P}(\mathcal{Z})}{\inf}\sum_{m=1}^{M}\alpha_{m}W^{2}_{2}(\mathbb{P}, \mathbb{P}_{\text{Enc}(X, C)|C}(\cdot|c_{m}); \lVert \cdot \rVert) + \sum_{m=1}^{M}\alpha_{m}\lVert \bar{c}-c_{m} \rVert^{2}.
\end{split}}}
\end{equation*}

\subsubsection{Proof of Theorem \ref{thm4}.}\label{appA.2.6}
We first derive a lemma for Theorem \ref{thm4}.
\begin{lemma}\label{lem4}Suppose the encoder, generator, and transport map satisfy conditions $(A1)$ through $(A3)$.
Then, 
\begin{equation}\label{s_eq:bounds}
\scalebox{0.99}{\ensuremath{\begin{split}
&(1/2)\sum_{m=1}^{M}\sum_{m'=1}^{M}\alpha_{m}\alpha_{m'}W^{2}_{2}(\mathbb{P}_{(X,C)|C}(\cdot|c_{m}), \mathbb{P}_{(X,C)|C}(\cdot|c_{m'}); d_{\text{Enc}})\\
&\leq \underset{\mathbb{P} \in \cup_{c \in \mathcal{C}}\mathcal{P}(\mathcal{X} \times \{c\})}{\inf}\sum_{m=1}^{M}\alpha_{m}W^{2}_{2}(\mathbb{P}, \mathbb{P}_{(X, C)|C}(\cdot|c_{m}); d_{\text{Enc}}) \\
&\leq \sum_{m=1}^{M} \alpha_{m}W^{2}_{2}(\mathbb{P}_{(\tilde{X}(\alpha_{1},\dots,\alpha_{M}), \bar{c})}, \mathbb{P}_{(X, C)|C}(\cdot|c_{m}); d_{\text{Enc}}) \\
&\leq (1/2)\sum_{m=1}^{M}\sum_{m'=1}^{M}\alpha_{m}\alpha_{m'}\int d^{2}_{\text{Enc}}((T(x_{1}, c_{1}, c_{m}), c_{m}), (T(x_{1}, c_{1}, c_{m'}), c_{m'})) d\mathbb{P}_{X|C}(x_{1}|c_{1})
\end{split}}}
\end{equation}
holds. When we further suppose the condition $(A4)$, $\mathbb{P}_{(\tilde{X}(\alpha_{1},\dots,\alpha_{M}), \bar{c})}$ is the Wasserstein barycenter of $(\mathbb{P}_{(X, C)|C}(\cdot|c_{m}))_{m=1}^{M}$ w.r.t. weights $(\alpha_{m})_{m=1}^{M}$ and all the terms in Equation (\ref{eq:bounds}) are equal to $\sum_{m=1}^{M}\alpha_{m}\lVert \bar{c} - c_{m}\rVert^{2}$.
\end{lemma}
\begin{proof}
First, we derive Equation (\ref{s_eq:bounds}). The second inequality is trivial by the definition of the infimum. For the first inequality, by Lemma \ref{lem1}, the definition of the Wasserstein distance, and Equation (\ref{eq:dual_multimarginal}),
\begin{equation*}
\begin{split}
& (1/2)\sum_{m=1}^{M}\sum_{m'=1}^{M}\alpha_{m}\alpha_{m'}W^{2}_{2}(\mathbb{P}_{(\text{Enc}(X, C), C)|C}(\cdot|c_{m}), \mathbb{P}_{(\text{Enc}(X, C), C)|C}(\cdot|c_{m'}); \lVert \cdot \rVert) \\
&\leq (1/2)\sum_{m=1}^{M}\sum_{m'=1}^{M}\alpha_{m}\alpha_{m'}\int \lVert (z_{m}, c_{m}) - (z_{m'}, c_{m'}) \rVert^{2} d\mathbb{Q}^{*}_{(Z_{m}, Z_{m'})}(z_{m}, c_{m}, z_{m'}, c_{m'})
\end{split}
\end{equation*}
where $\mathbb{Q}^{*}_{(Z_{1},\dots,Z_{M})}$ is the unique solution of $\underset{\pi \in \Pi(\mathbb{P}_{\text{Enc}(X, C)|C}(\cdot|c_{1}),\dots,\mathbb{P}_{\text{Enc}(X, C)|C}(\cdot|c_{M}))}{\inf} \int \sum_{m=1}^{M}\alpha_{m} \lVert z_{m} - \sum_{m=1}^{M}\alpha_{m}z_{m} \rVert^{2} \pi(z_{1},\dots,z_{M})$. Since 
\begin{equation}\label{eq:variance_equality}
(1/2)\sum_{m=1}^{M}\sum_{m'=1}^{M}\alpha_{m}\alpha_{m'}\lVert a_{m} - a_{m'} \rVert^{2}=\sum_{m=1}^{M}\alpha_{m}\lVert a_{m} - \sum_{m=1}^{M}\alpha_{m}a_{m} \rVert^{2}    
\end{equation}
holds for any sequence $(a_{m})_{m=1}^{M}$, we have that $(1/2)\sum_{m=1}^{M}\sum_{m'=1}^{M}\alpha_{m}\alpha_{m'}\int \lVert (z_{m}, c_{m}) - (z_{m'}, c_{m'}) \rVert^{2} d\mathbb{Q}^{*}_{(Z_{m}, Z_{m'})}(z_{m}, c_{m}, z_{m'}, c_{m'})$ is equal to $\sum_{m=1}^{M}\alpha_{m}\int \lVert z_{m} - \sum_{m=1}^{M}\alpha_{m}z_{m} \rVert^{2}d\mathbb{Q}^{*}_{(Z_{1},\dots,Z_{M})}(z_{1},\dots,z_{M})+\sum_{m=1}^{M}\alpha_{m} \lVert \bar{c} - c_{m} \rVert^{2}$. The first term in the RHS equals to $\underset{\pi \in \Pi(\mathbb{P}_{\text{Enc}(X, C)|C}(\cdot|c_{1}),\dots,\mathbb{P}_{\text{Enc}(X, C)|C}(\cdot|c_{M}))}{\inf} \int \sum_{m=1}^{M}\alpha_{m} \lVert z_{m} - \sum_{m=1}^{M}\alpha_{m}z_{m} \rVert^{2} \pi(z_{1},\dots,z_{M})$. This and Lemmas \ref{lem1} and \ref{lem2} imply
\begin{equation*}
\begin{split}
& (1/2)\sum_{m=1}^{M}\sum_{m'=1}^{M}\alpha_{m}\alpha_{m'}W^{2}_{2}(\mathbb{P}_{(\text{Enc}(X, C), C)|C}(\cdot|c_{m}), \mathbb{P}_{(\text{Enc}(X, C), C)|C}(\cdot|c_{m'}); \lVert \cdot \rVert) \\
& \leq \underset{\mathbb{P} \in \cup_{c \in \mathcal{C}}\mathcal{P}(\mathcal{Z} \times \{c\})}{\inf}\sum_{m=1}^{M}\alpha_{m}W^{2}_{2}(\mathbb{P}, \mathbb{P}_{(\text{Enc}(X, C), C)|C}(\cdot|c_{m}); \lVert \cdot \rVert) \\
& \leq \sum_{m=1}^{M} \alpha_{m}W^{2}_{2}(\mathbb{P}_{ (\sum_{m=1}^{M}\alpha_{m}\text{Enc}(T(X_{1}, c_{1}, c_{m}), c_{m}), \bar{c}), \bar{c}) }, \mathbb{P}_{(\text{Enc}(X, C), C)|C}(\cdot|c_{m}); \lVert \cdot \rVert),
\end{split}
\end{equation*}
which concludes the first inequality by Lemma \ref{lem1}. For the last inequality, by the definition of the Wasserstein distance, $W^{2}_{2}(\mathbb{P}_{ (\sum_{m=1}^{M}\alpha_{m}\text{Enc}(T(X_{1}, c_{1}, c_{m}), c_{m}), \bar{c}), \bar{c}) }, \mathbb{P}_{(\text{Enc}(X, C), C)|C}(\cdot|c_{m}); \lVert \cdot \rVert) \leq \int \lVert (\sum_{m=1}^{M}\alpha_{m}\text{Enc}(T(X_{1}, c_{1}, c_{m}), c_{m}), \bar{c}) - (T(x_{1}, c_{1}, c_{m'}), c_{m'}) \rVert^{2} d\mathbb{P}_{X|C}(x_{1}|c_{1})$. This, Equation (\ref{eq:variance_equality}), and Lemma \ref{lem1} conclude the proof for the last inequality.

Next, we derive that the distribution of the latent interpolation result is the Wasserstein barycenter when $(A4)$ holds. Since the Wasserstein variance of $(\mathbb{P}_{\text{Enc}(X,C)|C}(\cdot|c_{m}))_{m=1}^{M}$ is zero, by Lemmas \ref{lem1} and \ref{lem2}, it is sufficient to show that 
\begin{equation}\label{eq:barycenter}
    \sum_{m=1}^{M}\alpha_{m}W^{2}_{2}(\mathbb{P}_{\text{Gen}(\sum_{m=1}^{M}\alpha_{m}\text{Enc}(T(X_{1}, c_{1}, c_{m}), c_{m}), \bar{c})}, \mathbb{P}_{(X, C)|C}(\cdot|c_{m}); d_{\text{Enc}}) = \sum_{m=1}^{M}\alpha_{m}\lVert \bar{c}-c_{m} \rVert^{2}.
\end{equation}
By $(A4)$, the optimal transportation cost from $\mathbb{P}_{\text{Enc}(X, C)|C}(\cdot|c_{1})$ to $\mathbb{P}_{\text{Enc}(X, C)|C}(\cdot|c_{m})$ is zero, which implies that $\text{Enc}(X_{1}, c_{1})=\text{Enc}(T(X_{1}, c_{1}, c_{m}), c_{m})$ almost surely for all $m$. Thus, $\mathbb{P}_{\text{Gen}(\sum_{m=1}^{M}\alpha_{m}\text{Enc}(T(X_{1}, c_{1}, c_{m}), c_{m}))}=\mathbb{P}_{\text{Gen}(\text{Enc}(X_{1}, c_{1}), c_{m})}=\mathbb{P}_{\text{Gen}(\text{Enc}(X, C), c_{m})|C}(\cdot|c_{1})=\mathbb{P}_{\text{Gen}(\text{Enc}(X, C), c_{m})|C}(\cdot|c_{m})=\mathbb{P}_{X|C}(\cdot|c_{m})$ for all $m$. This implies that the LHS in Equation (\ref{eq:barycenter}) equals to $\sum_{m=1}^{M}\alpha_{m}W^{2}_{2}(\mathbb{P}_{(X, \bar{c})|C}(\cdot|c_{m}), \mathbb{P}_{(X, c_{m})|C}(\cdot|c_{m}); d_{\text{Enc}})=\sum_{m=1}^{M}\alpha_{m}\lVert \bar{c} - c_{m} \rVert^{2}$, which concludes the proof. For general $p$-Wasserstein distances, $ \sum_{m=1}^{M}\alpha_{m}W^{2}_{p}(\mathbb{P}_{\text{Gen}(\sum_{m=1}^{M}\alpha_{m}\text{Enc}(T(X_{1}, c_{1}, c_{m}), c_{m}), \bar{c})}, \mathbb{P}_{(X, C)|C}(\cdot|c_{m}); d_{\text{Enc}})=\sum_{m=1}^{M}\alpha_{m}W^{2}_{p}(\mathbb{P}_{(X, \bar{c})|C}(\cdot|c_{m}), \mathbb{P}_{(X, c_{m})|C}(\cdot|c_{m}); d_{\text{Enc}})=\sum_{m=1}^{M}\alpha_{m}\lVert \bar{c} - c_{m} \rVert^{2}$ holds.

Last, we derive that all the terms in Equation (\ref{s_eq:bounds}) are equal to $\sum_{m=1}^{M}\alpha_{m}\lVert \bar{c} - c_{m} \rVert^{2}$ when $(A4)$ holds. Since $\mathbb{P}_{\text{Gen}(\sum_{m=1}^{M}\alpha_{m}\text{Enc}(T(X_{1}, c_{1}, c_{m}), c_{m}))}=\mathbb{P}_{X|C}(\cdot|c_{m})$ for all $m$, derived in the above paragraph, implies that $\mathbb{P}_{(X, C)|C}(\cdot|c_{m'})=\mathbb{P}_{(X, c_{m'})|C}(\cdot|c_{m})$, the lower bound becomes $(1/2)\sum_{m=1}^{M}\sum_{m}^{M}\alpha_{m}\alpha_{m'}\lVert c_{m} - c_{m'} \rVert^{2}$. Similarly, the $\text{Enc}(X_{1}, c_{1})=\text{Enc}(T(X_{1}, c_{1}, c_{m}), c_{m})$ for all $m$ implies that the upper bound becomes $(1/2)\sum_{m=1}^{M}\sum_{m}^{M}\alpha_{m}\alpha_{m'}\lVert c_{m} - c_{m'} \rVert^{2}$. Now, Equation (\ref{eq:variance_equality}) concludes the proof.
\end{proof}
Now, we show Theorem \ref{thm4}. By Lemma \ref{lem4}, we have
\begin{equation*}
\scalebox{0.92}{\ensuremath{\begin{split}
& \sum_{m=1}^{M} \alpha_{m}W^{2}_{2}(\mathbb{P}_{(\tilde{X}(\alpha_{1},\dots,\alpha_{M}), \bar{c})}, \mathbb{P}_{(X, C)|C}(\cdot|c_{m}); d_{\text{Enc}}) - \underset{\mathbb{P} \in \cup_{c \in \mathcal{C}}\mathcal{P}(\mathcal{X} \times \{c\})}{\inf}\sum_{m=1}^{M}\alpha_{m}W^{2}_{2}(\mathbb{P}, \mathbb{P}_{(X, C)|C}(\cdot|c_{m}); d_{\text{Enc}})\\
& \leq \frac{1}{2} \sum_{m=1}^{M}\sum_{m'=1}^{M}\alpha_{m}\alpha_{m'}\big( \int d^{2}_{\text{Enc}}((T(x_{1}, c_{1}, c_{m}), c_{m}), (T(x_{1}, c_{1}, c_{m'}), c_{m'})) d\mathbb{P}_{X|C}(x_{1}|c_{1}) \\
& \qquad \qquad \qquad \qquad \qquad - W^{2}_{2}(\mathbb{P}_{(X,C)|C}(\cdot|c_{m}), \mathbb{P}_{(X,C)|C}(\cdot|c_{m'}); d_{\text{Enc}}) \big).
\end{split}}}
\end{equation*}
By (A3), $\int d^{2}_{\text{Enc}}((T(x_{1}, c_{1}, c_{m}), c_{m}), (T(T(x_{1}, c_{1}, c_{m}), c_{m}, c_{m'})) d\mathbb{P}_{X|C}(x_{1}|c_{1})$ is equal to $W^{2}_{2}(\mathbb{P}_{(X,C)|C}(\cdot|c_{m}), \mathbb{P}_{(X,C)|C}(\cdot|c_{m'}); d_{\text{Enc}})$, which concludes the proof of Equation (4). The upper bound is zero and the distribution of latent interpolation result is the Wasserstein barycenter by Lemma \ref{lem4}.

\subsubsection{Proof of Theorem \ref{thm5}.}\label{appA.2.7}
By $(A5)$, the data generation structure can be expressed as $X=\text{Gen}(Z, C)$ where $Z:=\text{Enc}(X, C)$ independent of $C$. The Wasserstein barycenter of $(\mathbb{P}_{(X, C)|C}(\cdot|c_{m}))_{m=1}^{M}=(\mathbb{P}_{(\text{Gen}(Z, C), C)|C}(\cdot|c_{m}))_{m=1}^{M}$ w.r.t. $W_{2}(\cdot, \cdot; d_{\text{Enc}})$ is the same as that of $(\mathbb{P}_{(Z, c_{m})})_{m=1}^{M}$ w.r.t. $W_{2}(\cdot, \cdot; \lVert \cdot \rVert)$, which implies that $\mathbb{P}_{X|C}(\cdot|\bar{c})$ is the Wasserstein barycenter. Now, Theorem \ref{thm4} concludes the proof.

\subsubsection{Proof of Theorem \ref{thm6}.}\label{appA.2.8}
We denote the first, second, and third term in objective (4) as follows.

\begin{itemize}
    \item $L_{\text{TransportCost}}(T):=\int d_{\text{Enc}}^{p}((x_{0}, c_{0}), (T(x_{0}, c_{0}, c_{1}), c_{1}))d\mathbb{P}_{X|C}(x_{0}|c_{0})d\mathbb{P}_{C}(c_{0})d\mathbb{P}_{C}(c_{1})$
    \item $L_{\text{MatchData}}(T):=W_{p}(\mathbb{P}_{(X, C)}\mathbb{P}_{C}, \mathbb{P}_{(T(X_{0}, C_{0}, C_{1}), C_{1}, C_{0})}; \| \cdot \|)$
    \item $L_{\text{Cycle}}(T):=\int \| x_{0} - T(T(x_{0}, c_{0}, c_{1}), c_{1}, c_{0}) \| d\mathbb{P}_{X|C}(x_{0}|c_{0})d\mathbb{P}_{C}(c_{0})d\mathbb{P}_{C}(c_{1})$
\end{itemize}
By definition of the optimal transport map, $L_{\text{MatchData}}(T^{\dag})=L_{\text{Cycle}}(T^{\dag})=0$ and $L_{\text{TransportCost}}(T^{\dag}) \leq L_{\text{TransportCost}}(T)$ for all $T: \mathcal{X} \times \mathcal{C}^{2} \to \mathcal{X}$. Thus, $T^{\dag}$ is a minimizer. Let $T$ be a minimizer of objective (4). Then, $L_{\text{MatchData}}(T)=L_{\text{Cycle}}(T)=0$ and $L_{\text{TransportCost}}(T) = L_{\text{TransportCost}}(T^{\dag})$. By definition of the optimal transport map, it implies $T(x_{0}, c_{0}, c_{1}):= T_{c_{0} \to c_{1}}^{\dag}(x_{0})$ with probability $1$ w.r.t. $\mathbb{P}_{X|C}(x_{0}|c_{0})$ for all $c_{0}, c_{1} \in \mathcal{C}$. Thus, $T=T^{\dag}$ and it concludes that $T^{\dag}$ is the unique minimizer.

\section{Details on Experiments}\label{appB}

\subsection{Implementation Details}\label{appB.1}
\begin{table}[ht]\label{tab:arc_enc_gen}
\caption{Architectures of encoder and generator networks.}
\centering
\resizebox{1.0\columnwidth}{!}{
\begin{tabular}{cc}
\toprule
Encoder                                                       & Generator                                                        \\ \hline
Conv. with kernel 11x11, filter size 128, stride 1, padding 5 & ConvTran. with kernel 4x4, filter size 1024, stride 1, padding 0 \\
BatchNormalization                                            & BatchNormalization                                               \\
LeakyReLU with slope 0.2                                      & LeakyReLU with slope 0.2                                         \\
Conv. with kernel 6x6, filter size 128, stride 2, padding 2   & ConvTran. with kernel 4x4, filter size 1024, stride 2, padding 1 \\
BatchNormalization                                            & BatchNormalization                                               \\
LeakyReLU with slope 0.2                                      & LeakyReLU with slope 0.2                                         \\
Conv. with kernel 4x4, filter size 256, stride 2, padding 1   & ConvTran. with kernel 4x4, filter size 512, stride 2, padding 1  \\
BatchNormalization                                            & BatchNormalization                                               \\
LeakyReLU with slope 0.2                                      & LeakyReLU with slope 0.2                                         \\
Conv. with kernel 4x4, filter size 512, stride 2, padding 1   & ConvTran. with kernel 4x4, filter size 256, stride 2, padding 1  \\
BatchNormalization                                            & BatchNormalization                                               \\
LeakyReLU with slope 0.2                                      & LeakyReLU with slope 0.2                                         \\
Conv. with kernel 4x4, filter size 1024, stride 2, padding 1  & ConvTran. with kernel 4x4, filter size 128, stride 2, padding 1  \\
BatchNormalization                                            & BatchNormalization                                               \\
LeakyReLU with slope 0.2                                      & LeakyReLU with slope 0.2                                         \\
Conv. with kernel 4x4, filter size 1024, stride 2, padding 1  & ConvTran. with kernel 6x6, filter size 128, stride 2, padding 2  \\
BatchNormalization                                            & BatchNormalization                                               \\
LeakyReLU with slope 0.2                                      & LeakyReLU with slope 0.2                                         \\
Conv. with kernel 4x4, filter size 128, stride 1, padding 0   & ConvTran. with kernel 11x11, filter size 1, stride 1, padding 5  \\
\multicolumn{1}{l}{}                                          & Sigmoid      \\
\bottomrule
\end{tabular}}
\end{table}

\begin{table}[ht]\label{tab:arc_trans}
\caption{The architecture of the transport map network.}
\centering
\resizebox{1.0\columnwidth}{!}{
\begin{tabular}{cc}
\toprule
Transport map                                                    & Residual block                                              \\ \hline
Conv. with kernel 11x11, filter size 64, stride 1, padding 5     & Conv. with kernel 3x3, filter size 256, stride 1, padding 1 \\
BatchNormalization                                               & BatchNormalization                                          \\
ReLU                                                             & ReLU                                                        \\
Conv. with kernel 4x4, filter size 128, stride 2, padding 1      & Conv. with kernel 3x3, filter size 256, stride 1, padding 1 \\
BatchNormalization                                               & BatchNormalization                                          \\
ReLU                                                             &                                                             \\
Conv. with kernel 4x4, filter size 256, stride 2, padding 1      &                                                             \\
BatchNormalization                                               &                                                             \\
ReLU                                                             &                                                             \\
Residual block 1                                                 &                                                             \\
Residual block 2                                                 &                                                             \\
Residual block 3                                                 &                                                             \\
Residual block 4                                                 &                                                             \\
Residual block 5                                                 &                                                             \\
Residual block 6                                                 &                                                             \\
ConvTrans. with kernel 4x4, filter size 128, stride 2, padding 1 &                                                             \\
BatchNormalization                                               &                                                             \\
ReLU                                                             &                                                             \\
ConvTrans. with kernel 4x4, filter size 64, stride 2, padding 1  &                                                             \\
BatchNormalization                                               &                                                             \\
ReLU                                                             &                                                             \\
ConvTrans. with kernel 11x11, filter size 1, stride 1, padding 5 &                                                             \\
Sigmoid                                                          &      \\  \bottomrule
\end{tabular}}
\end{table}

\begin{figure*}[ht]\label{fig:further_generation}
 \vskip 0.2in
  \centerline{
  \begin{minipage}[b]{0.45\textwidth}
      \begin{center}
        \includegraphics[width=.95\columnwidth]{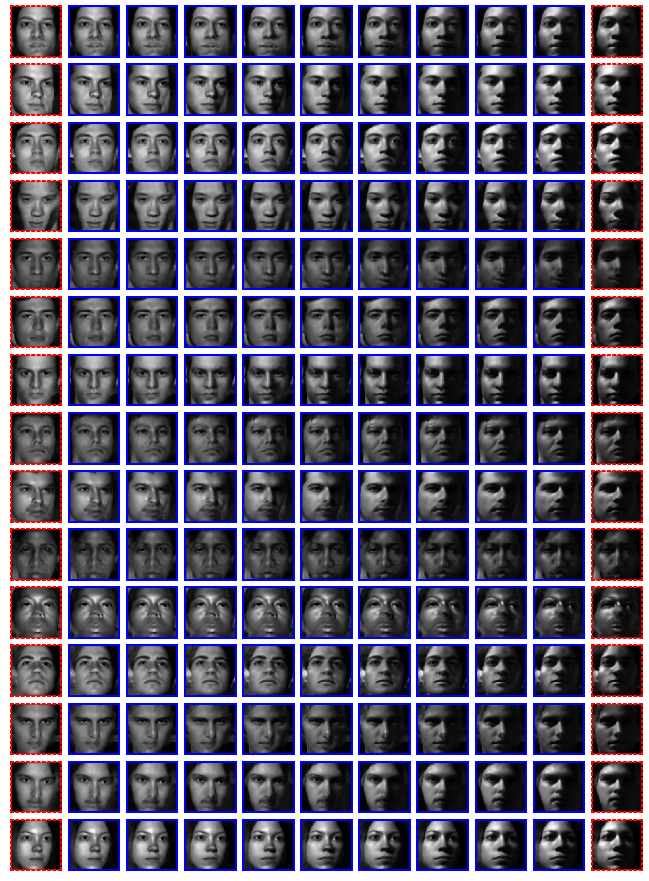}
      \end{center}
  \end{minipage}
  \begin{minipage}[b]{0.45\textwidth}
      \begin{center}
        \includegraphics[width=0.95\columnwidth]{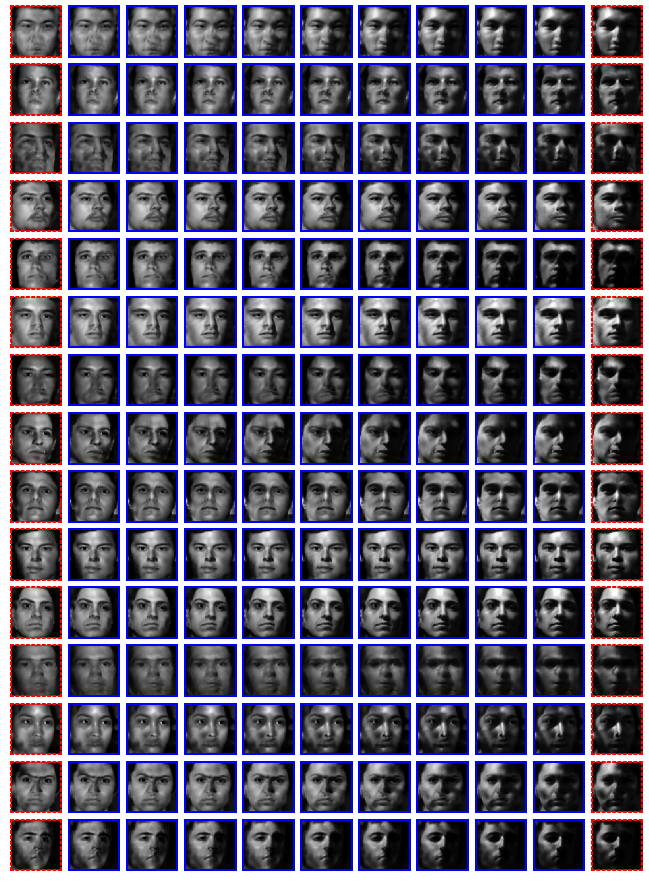}      
      \end{center}
  \end{minipage}
  }
  \caption{Conditional generation results by the proposed method (left) and cAAE (right). The proposed method produces face images with clearer eyes, noses, and mouths than baselines. For each method, the leftmost and rightmost columns show generation results for observed domains and intermediate columns show results for unobserved intermediate domains.}
  \vskip 0.2in
\end{figure*}

\begin{table}[ht]\label{tab:arc_disc_reg}
\caption{Architectures of discriminator for generator, discriminator for transport map, and auxiliary regressor.}
\centering
\resizebox{1.0\columnwidth}{!}{
\begin{tabular}{ccc}
\toprule
Discriminator for generator & Discriminator for transport map                              & Auxiliary regressor \\ \hline
Linear with filter size 512 & Conv. with kernel 4x4, filter size 64, stride 2, padding 1   & Conv. with kernel 4x4, filter size 64, stride 2, padding 1   \\
ReLU                        & LeakyReLU with slope 0.01                                    & LeakyReLU with slope 0.01                                    \\
Linear with filter size 512 & Conv. with kernel 4x4, filter size 128, stride 2, padding 1  & Conv. with kernel 4x4, filter size 128, stride 2, padding 1  \\
ReLU                        & LeakyReLU with slope 0.01                                    & LeakyReLU with slope 0.01                                    \\
Linear with filter size 512 & Conv. with kernel 4x4, filter size 256, stride 2, padding 1  & Conv. with kernel 4x4, filter size 256, stride 2, padding 1  \\
ReLU                        & LeakyReLU with slope 0.01                                    & LeakyReLU with slope 0.01                                    \\
Linear with filter size 512 & Conv. with kernel 4x4, filter size 512, stride 2, padding 1  & Conv. with kernel 4x4, filter size 512, stride 2, padding 1  \\
ReLU                        & LeakyReLU with slope 0.01                                    & LeakyReLU with slope 0.01                                    \\
Linear with filter size 1   & Conv. with kernel 4x4, filter size 1024, stride 2, padding 1 & Conv. with kernel 4x4, filter size 1024, stride 2, padding 1 \\
Sigmoid                     & LeakyReLU with slope 0.01                                    & LeakyReLU with slope 0.01                                    \\
                            & Conv. with kernel 4x4, filter size 2048, stride 2, padding 1 & Conv. with kernel 4x4, filter size 2048, stride 2, padding 1 \\
                            & LeakyReLU with slope 0.01                                    & LeakyReLU with slope 0.01                                    \\
                            & Conv. with kernel 3x3, filter size 1, stride 1, padding 1    & Conv. with kernel 3x3, filter size 2, stride 1, padding 1    \\
                            & Average pooling with kernel 2x2, stride 2                    & Average pooling with kernel 2x2, stride 2     \\
\bottomrule
\end{tabular}}
\end{table}

Our work consists of three main networks, encoder, generator, and transport map, and three auxiliary networks, discriminator for generator, discriminator for transport map, and auxiliary regressor. Architectures of the encoder and generator networks are adopted from DCGAN \citep{radford2015unsupervised}, and the architecture of the transport map is adopted from StarGAN \citep{choi2018stargan}. Architectures are modified to concatenate light conditions to latent variables. Table \hyperref[tab:arc_enc_gen]{1} shows architectures of encoder and generator networks. Conv and ConvTran denote convolutional layer and convolutional transpose layer, respectively. Table \hyperref[tab:arc_trans]{2} shows the architecture of the transport map network. We apply skip connection to input features in intermediate layers of convolutional layers into convolutional transpose layers. Table \hyperref[tab:arc_disc_reg]{3} shows architectures of discriminator for generator, discriminator for transport map, and auxiliary regressor.

We control the size of the networks of baselines for a fair comparison. For conditional AAE (cAAE, \citealp{makhzani2015adversarial}), architectures are the same as ours except encoder and discriminator input only latent variable. For CycleGAN \citep{zhu2017unpaired}, architectures are the same as ours. For StarGAN, architectures are the same as ours except translator map inputs only source data and target domain labels. The dimension of latent variable is $128$. For both the proposed method and baselines, we train the encoder and generator pair for $100,000$ iterations with batch size of $32$, and train the transport map for $50,000$ iterations with batch size of $16$. We use the Adam \citep{kingma2014adam} optimizer and set the initial learning rate to $0.0002$ and to linearly decrease to $0$ for the encoder and generator pair and to $0.0001$ for the transport map. In the first step of training encoder and generator pair, we update encoder and generator for every $5$ iterations while update discriminator for generator for every iteration. In the second step of training transport map, we update transport map and auxiliary regressor for every $5$ iterations while update discriminator for transport map for every iteration. For a data pre-processing, we apply a face detection algorithm proposed by \citet{viola2001rapid} to crop the face part. The resolution of image is scaled to $128$, the range of images is scaled to $[0, 1]$, and the horizontal flip with probability $0.5$ is applied during training.

We denote losses as follows and provide values of coefficients for losses. 
\begin{itemize}
    \item $L_{Recon}(\text{Enc}, \text{Gen}):=\int \| x - \text{Gen}(\text{Enc}(x, c), c) \| ^{p} d\mathbb{P}_{X|C}(x|c)d\mathbb{P}_{C}(c)$
    \item $L_{MatchLatent}(\text{Enc}, \text{Gen}):=\mathcal{D}_{\text{JS}} \bigg( \mathbb{P}_{Z}(z)\mathbb{P}_{C}(c), \big( \int \delta_{z}(\text{Enc}(x,c))d\mathbb{P}_{X|C}(x|c) \big) \mathbb{P}_{C}(c) \bigg)$
    \item $L_{ReconLatent}(\text{Enc}, \text{Gen}):=\int \| (z, c) - (\text{Enc}(\text{Gen}(z, c), c), c) \| d\mathbb{P}_{(Z, C)}(z, c)$
    \item $L_{\text{TransportCost}}(T):=\int d_{\text{Enc}}^{p}((x_{0}, c_{0}), (T(x_{0}, c_{0}, c_{1}), c_{1}))d\mathbb{P}_{X|C}(x_{0}|c_{0})d\mathbb{P}_{C}(c_{0})d\mathbb{P}_{C}(c_{1})$
    \item $L_{\text{MatchData}}(T):=W_{p}(\mathbb{P}_{(X, C)}\mathbb{P}_{C}, \mathbb{P}_{(T(X_{0}, C_{0}, C_{1}), C_{1}, C_{0})}; \| \cdot \|)$
    \item $L_{\text{Cycle}}(T):=\int \| x_{0} - T(T(x_{0}, c_{0}, c_{1}), c_{1}, c_{0}) \| d\mathbb{P}_{X|C}(x_{0}|c_{0})d\mathbb{P}_{C}(c_{0})d\mathbb{P}_{C}(c_{1})$
\end{itemize}
The coefficients of $L_{Recon}(\text{Enc}, \text{Gen})$, $L_{MatchLatent}(\text{Enc}, \text{Gen})$, $L_{ReconLatent}(\text{Enc}, \text{Gen})$, $L_{\text{MatchData}}(T)$, and $L_{\text{Cycle}}(T)$ are $100.0$, $1.0$, $0.1$, $1.0$, and $5.0$, respectively. For $L_{\text{TransportCost}}(T)$, we consider $\{0.1, 1.0, 10.0, 100.0\}$ and choose the model yielding the best validation FID scores. The coefficient of gradient penalty loss, regression loss, and reconsturction error in the second step, $\int d(x, T(x, c, c)) d\mathbb{P}_{(X,C)}(x,c)$, is $100.0$, $1.0$, and $10.0$, respectively. For CycleGAN, the coefficient of identity mapping loss is $1.0$. We extended the definition of $d_{\text{Enc}}$ by introducing a hyperparameter $\epsilon$ where the extended formula is $d_{\text{Enc}}((x_{1},c_{1}), (x_{2},c_{2})):= \big( \lVert \text{Enc}(x_{1}, c_{1}) - \text{Enc}(x_{2}, c_{2}) \rVert^{2} + \epsilon \lVert c_{1}-c_{2} \rVert^{2} \big)^{1/2}$ to balance distances on $\mathcal{Z}$ and on $\mathcal{C}$.

\subsection{Further Results}\label{appB.2}

Figure \hyperref[fig:further_generation]{4} presents further conditional generation results for unobserved intermediate domains. As in Figure 1, the proposed method produces face images with clearer eyes, noses, and mouths than baselines.

\begin{figure}[ht]
\vskip 0.2in
\begin{center}
\centerline{\includegraphics[width=0.8\columnwidth]{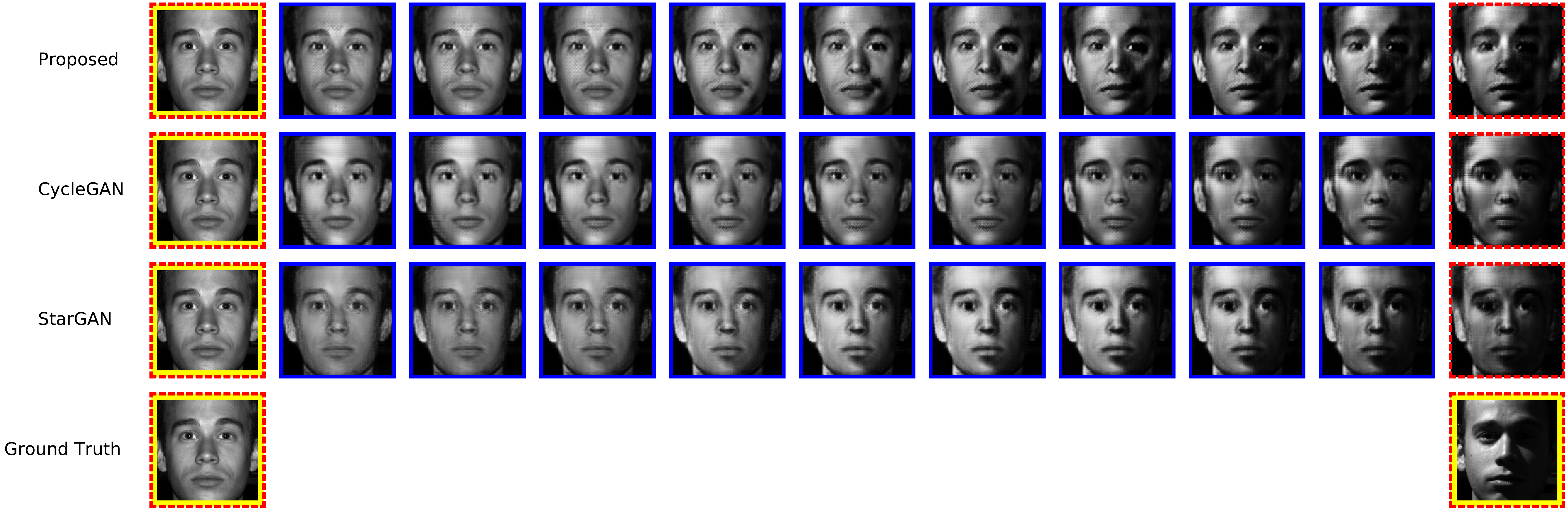}}
\caption{A visualization of latent interpolation with the real data and their translation results by various methods. The bottom row shows real images from two observed domains. From the first to third rows, the leftmost column shows the ground-truth, the rightmost column shows transportation results of the ground-truth, and intermediate columns show latent interpolation results for unobserved intermediate domains.}
\label{fig:further_latentinterpolation_real}
\end{center}
\vskip -0.2in
\end{figure}

Figure \hyperref[fig:further_latentinterpolation_real]{5} presents further results of latent interpolation on unobserved intermediate domains with a real image and its translation result. The bottom row shows the ground-truth images of fixed subject and pose from two observed domains with (azimuth, elevation) of $(0, 0)$, for the leftmost, and $(70, 45)$, for the rightmost. As in Figure 3 of the manuscript, the outputs of the proposed method are visually sharper and more plausible than the baselines.

\subsection{Computing Infrastructure}\label{appB.3}
We use about one hundred CPU cores and ten GPUs (five GeForce GTX 1080, two TITAN X, and three TITAN V) for experiments. A full training of the proposed method requires about \(15\) GPU hours for the encoder and generator pair and \(14\) GPU hours for the transport map.



\vskip 0.2in
\bibliography{ref}

\end{document}